\documentclass{article} 
\usepackage{macros}
\usepackage{hyperref}
\usepackage{authblk}
\usepackage{amsthm}
\newcommand{\keywords}[1]{{\bf Keywords: } #1\par}
\newtheorem{theorem}{Theorem}

\newtheorem{remark}{Remark}
\newtheorem{proposition}{Proposition}
\newtheorem{lemma}[theorem]{Lemma}

\title{Variance-Constrained Actor-Critic Algorithms for Discounted and Average Reward MDPs}

\author[1]{Prashanth L A \thanks{prashanth.la@inria.fr}}
\author[2]{Mohammad Ghavamzadeh \thanks{mohammad.ghavamzadeh@inria.fr}}
\affil[1]{\small INRIA Lille - Nord Europe, Team SequeL, FRANCE.}
\affil[2]{\small Adobe Research \& INRIA Lille - Team SequeL\footnote{currently at Adobe Research, on leave of absence from INRIA.}}
\date{}

\begin{document}

\maketitle

\begin{abstract} 
In many sequential decision-making problems we may want to manage risk by minimizing some measure of variability in rewards in addition to maximizing a standard criterion. Variance related risk measures are among the most common risk-sensitive criteria in finance and operations research. However, optimizing many such criteria is known to be a hard problem. In this paper, we consider both discounted and average reward Markov decision processes. For each formulation, we first define a measure of variability for a policy, which in turn gives us a set of risk-sensitive criteria to optimize. For each of these criteria, we derive a formula for computing its gradient. 
We then devise actor-critic algorithms that operate on three timescales - a TD critic on the fastest timescale, a policy gradient (actor) on the intermediate timescale, and a dual ascent for Lagrange multipliers on the slowest timescale. In the discounted setting, we point out the difficulty in estimating the gradient of the variance of the return and incorporate simultaneous perturbation approaches to alleviate this. The average setting, on the other hand, allows for an actor update using compatible features to estimate the gradient of the variance. We establish the convergence of our algorithms to locally risk-sensitive optimal policies. Finally, we demonstrate the usefulness of our algorithms in a traffic signal control application.  

\end{abstract} 


\section{Introduction}
\label{sec:introduction}

The usual optimization criteria for an infinite horizon Markov decision process (MDP) are the {\em expected sum of discounted rewards} and the {\em average reward}~\cite{puterman1994markov,BertsekasDP01}. Many algorithms have been developed to maximize these criteria both when the model of the system is known (planning) and unknown (learning)~\cite{BertsekasT96,sutton1998reinforcement}. These algorithms can be categorized to {\bf value function-based} methods that are mainly based on the two celebrated dynamic programming algorithms {\em value iteration} and {\em policy iteration}; and {\bf policy gradient} methods that are based on updating the policy parameters in the direction of the gradient of a performance measure, i.e.,~the value function of the initial state or the average reward. Policy gradient methods estimate the gradient of the performance measure either without using an explicit representation of the value function (e.g.,~\cite{Williams92SS,Marbach98SM,Baxter01IP}) or using such a 
representation in which case they are referred to as {\em actor-critic} algorithms (e.g.,~\cite{Sutton00PG,Konda00AA,Peters05NA,Bhatnagar07IN,bhatnagar2009natural}). Using an explicit representation for value function (e.g.,~linear function approximation) by actor-critic algorithms reduces the variance of the gradient estimate with the cost of adding it a bias. 

Actor-critic methods were among the earliest to be investigated in RL~\cite{Barto83NE,Sutton84TC}. They comprise a family of reinforcement learning (RL) methods that maintain two distinct algorithmic components: An {\em Actor}, whose role is to maintain and update an action-selection policy; and a {\em Critic}, whose role is to estimate the value function associated with the actor's policy. Thus, the critic addresses a problem of {\em prediction}, whereas the actor is concerned with {\em control}. A common practice is to update the policy parameters using stochastic gradient ascent, and to estimate the value-function using some form of temporal difference (TD) learning~\cite{Sutton88LP}. 

However in many applications, we may prefer to minimize some measure of {\em risk} as well as maximizing a usual optimization criterion. In such cases, we would like to use a criterion that incorporates a penalty for the {\em variability} induced by a given policy. This variability can be due to two types of uncertainties: {\bf 1)} uncertainties in the model parameters, which is the topic of {\em robust} MDPs (e.g.,~\cite{Nilim05RC,Delage10PO,Xu12DR}), and {\bf 2)} the inherent uncertainty related to the stochastic nature of the system, which is the topic of {\em risk-sensitive} MDPs (e.g.,~\cite{Howard72RS,Sobel82VD,filar1989variance}). 

In risk-sensitive sequential decision-making, the objective is to maximize a risk-sensitive criterion such as the expected exponential utility~\cite{Howard72RS}, a variance related measure~\cite{Sobel82VD,filar1989variance}, the percentile performance~\cite{Filar95PP}, or conditional value-at-risk (CVaR)~\cite{Ruszczynski10RA,Shen13RS}. 
Unfortunately, when we include a measure of risk in our optimality criteria, the corresponding optimal policy is usually no longer Markovian stationary (e.g.,~\cite{filar1989variance}) and/or computing it is not tractable (e.g.,~\cite{filar1989variance,Mannor11MV}). Although risk-sensitive sequential decision-making has a long history in operations research and finance, it has only recently grabbed attention in the machine learning community. Most of the work on this topic (including those mentioned above) has been in the context of MDPs (when the model of the system is known) and much less work has been done within the reinforcement learning (RL) framework (when the model is unknown and all the information about the system is obtained from the samples resulted from the agent's interaction with the environment). In risk-sensitive RL, we can mention the work by Borkar~\cite{Borkar01SF,Borkar02QR,
borkar2010learning} and Basu et al.~\cite{Basu08LA} who considered the expected exponential utility, the one by Mihatsch and Neuneier~\cite{Mihatsch02RS} that formulated a new risk-sensitive control framework based on transforming the temporal difference errors that occur during learning, and the one by Tamar et al.~\cite{tamar2012policy} on several variance related measures. Tamar et al.~\cite{tamar2012policy} study stochastic shortest path problems, and in this context, propose a policy gradient algorithm (and in a more recent work~\cite{Tamar13VA} an actor-critic algorithm) for maximizing several risk-sensitive criteria that involve both the expectation and variance of the {\em return} random variable (defined as the sum of the rewards that the agent obtains in an episode). 

In this paper,\footnote{This paper is an extension of an earlier work by the authors~\cite{Prashanth13AC} and includes novel second order methods in the discounted setting, detailed proofs of all proposed algorithms, and additional experimental results.} we develop actor-critic algorithms for optimizing variance-related risk measures in both discounted and average reward MDPs. In the following, we first summarize our contributions in the discounted reward setting and follow it with those in average reward setting.

\paragraph{\textbf{Discounted reward setting.}}
Here we define the measure of variability as the {\em variance of the return} (similar to~\cite{tamar2012policy}). We formulate the following constrained optimization problem with the aim of maximizing the mean of the return subject to its variance being bounded from above:
For a given $\alpha >0$,
\begin{align*}
\max_\theta V^\theta(x^0)\quad\quad \text{subject to} \quad\quad \Lambda^\theta(x^0)\leq\alpha.
\end{align*}
In the above, $V^\theta(x^0)$ is the mean of the return, starting in state $x^0$ for a policy identified by its parameter $\theta$, while $\Lambda^\theta(x^0)$ is the variance of the return (see Section \ref{sec:discounted-setting} for precise definitions).
A standard approach to solve the above problem is to employ the Lagrangian relaxation procedure~\cite{bertsekas1999nonlinear} and solve the following unconstrained problem:
\begin{align*}
\max_\lambda\min_\theta\left(L(\theta,\lambda) \stackrel{\triangle}{=} -V^\theta(x^0)+\lambda\big(\Lambda^\theta(x^0)-\alpha\big)\right),
\end{align*}
where $\lambda$ is the Lagrange multiplier. For solving the above problem, it is required to derive a formula for the gradient of the Lagrangian $L(\theta,\lambda)$, both w.r.t. $\theta$ and $\lambda$. While the gradient w.r.t. $\lambda$ is particularly simple since it is the constraint value, the other gradient, i.e., w.r.t. $\theta$ is complicated. We derive this formula in Lemma \ref{grad-V-U} and show that $\nabla_\theta L(\theta,\lambda)$ requires the gradient of the value function at every state of the MDP (see the discussion in Sections~\ref{sec:discounted-setting} and~\ref{sec:discounted-alg}). 

Note that we operate in a \textit{simulation optimization} setting, i.e., we have access to reward samples from the underlying MDP. Thus, it is required to estimate the mean and varaince of the return (we use a TD-critic for this purpose) and then use these estimates to compute gradient of the Lagrangian. The latter is used then used to descend in the policy parameter. 
We estimate the gradient of the Lagrangian using two simultaneous perturbation methods: {\em simultaneous perturbation stochastic approximation} (SPSA)~\cite{Spall92MS} and {\em smoothed functional} (SF)~\cite{katkovnik1972convergence}, resulting in two separate discounted reward actor-critic algorithms. In addition, we also propose second-order algorithms with a Newton step, using both SPSA and SF.

Simultaneous perturbation methods have been popular in the field of stochastic optimization and the reader is referred to~\cite{Bhatnagar13SR} for a textbook introduction. First introduced in~\cite{Spall92MS}, the idea of SPSA is to perturb each coordinate of a parameter vector uniformly using Rademacher random variable, in the quest for finding the minimum of a function that is only observable via simulation. Traditional gradient schemes require $2\kappa_1$ evaluations of the function, where $\kappa_1$ is the parameter dimension. On the other hand, SPSA requires only two evaluations  irrespective of the parameter dimension and hence is an efficient scheme, especially useful in high-dimensional settings. While a one-simulation variant of SPSA was proposed in \cite{spall1997one}, the original two-simulation SPSA algorithm is preferred as it is more efficient and also seen to work better than its one-simulation variant. Later enhancements to the original SPSA scheme include using deterministic perturbation using certain Hadamard matrices~\cite{
bhatnagar2003two} and second-order methods that estimate Hessian using SPSA~\cite{spall2000adaptive,bhatnagar2005adaptive}.
The SF schemes are another class of simultaneous perturbation methods, which again perturb each coordinate of the parameter vector uniformly. However, unlike SPSA, Gaussian random variables are used here for the perturbation. Originally proposed in~\cite{katkovnik1972convergence}, the SF schemes have been studied and enhanced in later works such as~\cite{styblinski1986,bhatnagar2007adaptive}. Further,~\cite{shalabh2011constrained} proposes both SPSA and SF like schemes for constrained optimization.

\paragraph{\textbf{Average reward setting.}} Here we first define the measure of variability as the {\em long-run variance} of a policy as follows:
\begin{align*}
 \Lambda(\theta) = \lim_{T\rightarrow\infty}\frac{1}{T}\E\left[\sum_{n=0}^{T-1}\big(R_n-\rho(\mu)\big)^2\mid \theta\right],
\end{align*}
where $\rho(\theta)$ is the average reward under policy identified by its parameter $\theta$ (see Section \ref{sec:average-setting} for precise definitions). The aim here is to solve the following constrained optimization problem:
\begin{align*}
 \max_\theta\rho(\theta)\quad\quad \text{subject to} \quad\quad \Lambda(\theta)\leq\alpha.
\end{align*}
As in the discounted setting. we derive an expression for the gradient of the Lagrangian (see Lemma \ref{grad-rho-eta}). Unlike the discounted setting, we do not require sophisticated simulation optimizations schemes, as the gradient expressions in Lemma \ref{grad-rho-eta} suggest a simpler alternative that employs \textit{compatible features}~\cite{Sutton00PG,Peters05NA}. Compatible features for linearly approximating the action-value function of policy $\theta$ are of the form $\nabla_\theta\log\mu(a|x)$. These features are well-defined if the policy is differentiable w.r.t.~its parameters $\theta$.~\citet{Sutton00PG} showed the advantages of using these features in approximating the action-value function in actor-critic algorithms. In \cite{bhatnagar2009natural}, the authors use compatible features to develop actor-critic algorithms for a risk-neutral setting. 
We extend this to variance-constrained setting and establish that square value function itself serves as a good baseline level when calculating the gradient of the average square reward (see the discussion surrounding Lemma \ref{TD-error-Advantage}). This facilitates the usage of   \textit{compatible features} for obtaining unbiased estimates of both average reward as well as square reward.   
We then develop an actor-critic algorithm that employ these \textit{compatible features} in order to descend in the policy parameter $\theta$ and also identify the bias that arises due to function approximation (see Lemma \ref{bias-average}).

\paragraph{\textbf{Proof of convergence.}} Using the ordinary differential equations (ODE) approach, we establish the asymptotic convergence of our algorithms to locally risk-sensitive optimal policies. Our algorithms employ multi-timescale stochastic approximation, in both settings. The convergence proof proceeds by analysing each timescale separately. In essence, the iterates on a faster timescale view those on a slower timescale as quasi-static, while the slower timescale iterate views that on a faster timescale as equilibrated. Using this principle, we show that TD critic (on the fastest timescale in all the algorithms) converge to fixed points of the Bellman operator, for any fixed policy $\theta$ and Lagrange multiplier $\lambda$. Next, for any given $\lambda$, the policy update tracks in the asymptotic limit and converges to the equilibria of the corresponding ODE. Finally, $\lambda$ updates on slowest timescale converge and the overall convergence is to a local saddle point of the Lagrangian. Moreover, the limiting point is feasible for the constrained optimization problem mentioned above, i.e., the policy obtained upon convergence satisfies the constraint that the variance is upper-bounded by $\alpha$.

\paragraph{\textbf{Simulation experiments.}} We demonstrate the usefulness of our discounted and average reward risk-sensitive actor-critic algorithms in a traffic signal control application. The objective in our formulation is to minimize the total number of vehicles in the system, which indirectly minimizes the delay experienced by the system. The motivation behind using a risk-sensitive control strategy is to reduce the variations in the delay experienced by road users. From the results, we observe that the risk-sensitive algorithms proposed in this paper result in a long-term (discounted or average) cost that is higher than their risk-neutral variants. However, from the empirical variance of the cost (both discounted as well as average) perspective, the risk-sensitive algorithms outperform their risk-neutral variants.


\begin{remark}
It is important to note that our both discounted and average reward algorithms can be easily extended to other variance related risk criteria such as the Sharpe ratio, which is popular in financial decision-making~\cite{Sharpe66MF} (see Remarks~\ref{subsec:discount-SR} and~\ref{subsec:average-SR} for more details). 
\end{remark}

\begin{remark}
Another important point is that the {\em expected exponential utility} risk measure can be also considered as an approximation of the mean-variance tradeoff due to the following Taylor expansion (see e.g.,~Eq.~11~in~\cite{Mihatsch02RS})
\begin{equation*}
-\frac{1}{\beta}\log\mathbb{E}[e^{-\beta X}] = \mathbb{E}[X] - \frac{1}{\beta}\text{Var}[X]+O(\beta^2),
\end{equation*}
and we know that it is much easier to design actor-critic or other reinforcement learning algorithms~\cite{Borkar01SF,Borkar02QR,Basu08LA,borkar2010learning} for this risk measure than those that will be presented in this paper. However, this formulation is limited in the sense that it requires knowing the ideal tradeoff between the mean and variance, since it takes $\beta$ as an input. On the other hand, the mean-variance formulations considered in this paper are more general because \\
\begin{inparaenum}[\bfseries (1)]
\item we optimize for the Lagrange multiplier $\lambda$, which plays a similar role to $\beta$, as a tradeoff between the mean and variance, and \\
\item it is usually more natural to know an upper-bound on the variance (as in the mean-variance formulations considered in this paper) than knowing the ideal tradeoff between the mean and variance (as considered in the expected exponential utility formulation). \\
\end{inparaenum}
Despite all these, we should not consider these formulations as replacement for each other or try to find a formulation that is the best for all problems, but instead should consider them as different formulations that each might be the right fit for a specific problem. 
\end{remark}

\paragraph{\textbf{Closely related works.}}
In comparison to~\cite{tamar2012policy} and \cite{Tamar13VA}, which are the most closely related contributions, we would like to point out the following:\\
  \begin{inparaenum}[\bfseries(1)]
   \item The authors develop policy gradient and actor-critic methods for stochastic shortest path problems in \cite{tamar2012policy} and \cite{Tamar13VA}, respectively. On the other hand, we devise actor-critic algorithms for both discounted and average reward MDP settings.; and \\
   \item More importantly, we note the difficulty in the discounted formulation that requires to estimate the gradient of the value function at every state of the MDP and also sample from two different distributions. This precludes us from using {\em compatible features} - a method that has been employed successfully in actor-critic algorithms in a risk-neutral setting (cf. \cite{bhatnagar2009natural}) as well as more recently in \cite{Tamar13VA} for a risk-sensitive stochastic shortest path setting. We alleviate the above mentioned problems for the discounted setting by employing simultaneous perturbation based schemes for estimating the gradient in the first order methods and Hessian in the second order methods, that we propose.\\ 
  \item Unlike \cite{tamar2012policy,Tamar13VA} who consider a fixed $\lambda$ in their constrained formulations, we perform dual ascent using sample variance constrants and optimize the Lagrange multiplier $\lambda$. In rigorous terms, $\lambda_n$ in our algorithms is shown to converge to a local maxima of $\nabla_\lambda L(\theta^{\lambda},\lambda)$ (here $\theta^\lambda$ is the limit of the $\theta$ recursion for a given value of $\lambda$) and the limit $\lambda^*$ is such that the variance constraint is satisfied for the corresponding policy $\theta^{\lambda^*}$.   
  \end{inparaenum}

\paragraph{\textbf{Organization of the paper.}}
The rest of the paper is organized as follows: 
In Section \ref{sec:preliminaries}, we describe the RL setting.
In Section \ref{sec:discounted-setting}, we describe the risk-sensitive MDP in the discounted setting and propose actor-critic algorithms for this setting in Section \ref{sec:discounted-alg}. In Section \ref{sec:average-setting}, we present the risk measure for the average setting and propose an actor-critic algorithm that optimizes this risk measure in Section \ref{sec:average-alg}. In Sections \ref{sec:SPSA-SF-proofs}--\ref{sec:average-analysis}, we present the convergence proofs for the algorithms in discounted and average reward settings, respectively. In Section \ref{sec:simulation}, we describe the experimental setup and present the results in both average and discounted cost settings. Finally, in Section \ref{sec:conclusions}, we provide the concluding remarks and outline a few future research directions.


\section{Preliminaries} 
\label{sec:preliminaries}

We consider sequential decision-making tasks that can be formulated as a reinforcement learning (RL) problem. In RL, an agent interacts with a dynamic, stochastic, and incompletely known environment, with the goal of optimizing some measure of its {\em long-term} performance. This interaction is often modeled as a Markov decision process (MDP). A MDP is a tuple $(\X,\A,R,P,x^0)$ where $\X$ and $\A$ are the state and action spaces; $R(x,a), x\in \X, a\in \A$ is the reward random variable whose expectation is denoted by $r(x,a)=\E\big[R(x,a)\big]$; $P(\cdot|x,a)$ is the transition probability distribution; and $x^0 \in \X$ is the initial state\footnote{Our algorithms can be easily extended to a setting where the initial state is determined by a distribution.}. We assume that both state and action spaces are finite. 

The rule according to which the agent acts in its environment (selects action at each state) is called a {\em policy}. A Markovian stationary policy $\mu(\cdot|x)$ is a probability distribution over actions, conditioned on the current state $x$. The goal in a RL problem is to find a policy that optimizes the long-term performance measure of interest, e.g.,~maximizes the {\em expected discounted sum of rewards} or the {\em average reward}.

In policy gradient and actor-critic methods, we define a class of parameterized stochastic policies $\big\{\mu(\cdot|x;\theta),x\in\X,\theta\in\Theta\subseteq\R^{\kappa_1}\big\}$, estimate the gradient of the performance measure w.r.t.~the policy parameters $\theta$ from the observed system trajectories, and then improve the policy by adjusting its parameters in the direction of the gradient. Since in this setting a policy $\mu$ is represented by its $\kappa_1$-dimensional parameter vector $\theta$, policy dependent functions can be written as a function of $\theta$ in place of $\mu$. So, we use $\mu$ and $\theta$ interchangeably in the paper. 

We make the following assumptions on the policy, parameterized by $\theta$:\\

\noindent
{\bf (A1)} {\em For any state-action pair $(x,a)\in\X\times\A$, the policy $\mu(a|x;\theta)$ is continuously differentiable in the parameter $\theta$.}

\noindent
{\bf (A2)} {\em The Markov chain induced by any policy $\theta$ is irreducible.} \\

The above assumptions are standard requirements in policy gradient and actor-critic methods.

Finally, we denote by $d^\mu(x)$ and $\pi^\mu(x,a)=d^\mu(x)\mu(a|x)$, the stationary distribution of state $x$ and state-action pair $(x,a)$ under policy $\mu$, respectively. The stationary distributions can be seen to exist because we consider a finite state-action space setting and irreducibility here implies positive recurrence. Similarly in the discounted formulation, we define the $\gamma$-discounted visiting distribution of state $x$ and state-action pair $(x,a)$ under policy $\mu$ as $d^\mu_\gamma(x|x^0)=(1-\gamma)\sum_{n=0}^\infty\gamma^n\Pr(x_n=x|x_0=x^0;\mu)$ and $\pi^\mu_\gamma(x,a|x^0)=d^\mu_\gamma(x|x^0)\mu(a|x)$.


\section{Discounted Reward Setting} 
\label{sec:discounted-setting}

For a given policy $\mu$, we define the return of a state $x$ (state-action pair $(x,a)$) as the sum of discounted rewards encountered by the agent when it starts at state $x$ (state-action pair $(x,a)$) and then follows policy $\mu$, i.e.,
\begin{align*}
D^\mu(x)&=\sum_{n=0}^\infty\gamma^nR(x_n,a_n)\mid x_0=x,\;\mu, \\
D^\mu(x,a)&=\sum_{n=0}^\infty\gamma^nR(x_n,a_n)\mid x_0=x,\;a_0=a,\;\mu.
\end{align*}
The expected value of these two random variables are the value and action-value functions of policy $\mu$, i.e.,
\begin{equation*}
V^\mu(x)=\E\big[D^\mu(x)\big] \quad\quad\quad \text{and} \quad\quad\quad Q^\mu(x,a)=\E\big[D^\mu(x,a)\big]. 
\end{equation*}
The goal in the standard (risk-neutral) discounted reward formulation is to find an optimal policy $\mu^*=\argmax_\mu V^\mu(x^0)$, where $x^0$ is the initial state of the system. 

The most common measure of the {\em variability} in the stream of rewards is the {\em variance of the return}, defined by
\begin{align}
\label{eq:V1}
\Lambda^\mu(x)&\stackrel{\triangle}{=}\E\big[D^\mu(x)^2\big]-V^\mu(x)^2=U^\mu(x)-V^\mu(x)^2. 
\end{align}
The above measure was first introduced by Sobel~\cite{Sobel82VD}. Note that 
\begin{equation*}
U^\mu(x) \stackrel{\triangle}{=} \E\left[D^\mu(x)^2\right]
\end{equation*}
is the {\em square reward value function} of state $x$ under policy $\mu$. On similar lines, we define the {\em square reward action-value function} of state-action pair $(x,a)$ under policy $\mu$ as 
\begin{equation*}
W^\mu(x,a) \stackrel{\triangle}{=} \E\left[D^\mu(x,a)^2\right].
\end{equation*}
From the Bellman equation of  $\Lambda^\mu(x)$, proposed by Sobel~\cite{Sobel82VD}, it is straightforward to derive the following Bellman equations for  $U^\mu(x)$ and $W^\mu(x,a)$:
\begin{align}
&U^\mu(x)=\sum_a\mu(a|x) r(x,a)^2+\gamma^2\sum_{a,x'}\mu(a|x)P(x'|x,a)U^\mu(x')+2\gamma\sum_{a,x'}\mu(a|x)P(x'|x,a)r(x,a)V^\mu(x'),  \label{eq:U-W-Bellman}\\
&W^\mu(x,a)=r(x,a)^2+\gamma^2\sum_{x'}P(x'|x,a)U^\mu(x') +2\gamma r(x,a)\sum_{x'}P(x'|x,a)V^\mu(x').\nonumber
\end{align}
Although $\Lambda^\mu$ of \eqref{eq:V1} satisfies a Bellman equation, unfortunately, it lacks the monotonicity property of dynamic programming (DP), and thus, it is not clear how the related risk measures can be optimized by standard DP algorithms~\cite{Sobel82VD}. Policy gradient and actor-critic algorithms are good candidates to deal with this risk  measure. 

We consider the following risk-sensitive measure for discounted MDPs: For a given $\alpha >0$,
\begin{equation}
\label{eq:discounted-risk-measure}
\max_\theta V^\theta(x^0)\quad\quad \text{subject to} \quad\quad \Lambda^\theta(x^0)\leq\alpha.
\end{equation}
Assuming that there is at least one policy (in the class of parameterized policies that we consider) that satisfies the variance constraint above, it can be inferred from Theorem 3.8 of \cite{altman1999constrained} that there exists an optimal policy that uses at most one randomization.

It is important to note that the algorithms proposed in this paper can be used for any risk-sensitive measure that is based on the variance of the return such as 
\begin{enumerate}
\item $\min_\theta \Lambda^\theta(x^0) \quad\quad$ subject to $\quad\quad V^\theta(x^0)\geq\alpha$,
\item $\max_\theta V^\theta(x^0)-\alpha\sqrt{\Lambda^\theta(x^0)}$, 
\item Maximizing the Sharpe Ratio, i.e.,~$\;\max_\theta V^\theta(x^0)/\sqrt{\Lambda^\theta(x^0)}$. Sharpe Ratio (SR) is a popular risk measure in financial decision-making~\cite{Sharpe66MF}. Section~\ref{subsec:discount-SR} presents extensions of our proposed discounted reward algorithms to optimize the Sharpe ration.
\end{enumerate}
To solve \eqref{eq:discounted-risk-measure}, we employ the Lagrangian relaxation procedure~\cite{bertsekas1999nonlinear} to convert it to the following unconstrained problem:  
\begin{equation}
\label{eq:unconstrained-discounted-risk-measure}
\max_\lambda\min_\theta\left(L(\theta,\lambda) \stackrel{\triangle}{=} -V^\theta(x^0)+\lambda\big(\Lambda^\theta(x^0)-\alpha\big)\right),
\end{equation}
where $\lambda$ is the Lagrange multiplier. The goal here is to find the saddle point of  $L(\theta,\lambda)$, i.e.,~a point  $(\theta^*,\lambda^*)$ that satisfies  
$$L(\theta, \lambda^*) \ge L(\theta^*, \lambda^*) \ge L(\theta^*, \lambda),\forall\theta\in\Theta,\forall \lambda>0.$$ For a standard convex optimization problem with mild regularity conditions, one can ensure the existence of a unique saddle point. Further, convergence to this point can be achieved by descending in  $\theta$ and ascending in  $\lambda$ using $\nabla_\theta L(\theta,\lambda)$ and $\nabla_\lambda L(\theta,\lambda)$, respectively.

However, we operate in a {\em simulation optimization} setting, where
\begin{inparaenum}[\bfseries (i)]
\item only sample estimates of the Lagrangian are observed; and
\item  the objective (Lagrangian) is not necessarily convex in $\theta$ (or there is no unique saddle point).
\end{inparaenum} 
Hence, performing primal descent and dual ascent, one can only get to a local saddle point, i.e., a tuple  $(\theta^*, \lambda^*)$ which is a local minima w.r.t. $\theta$ and local maxima w.r.t $\lambda$ of the Lagrangian.  

In our setting, the necessary gradients of the Lagrangian are as follows:
\begin{align*}
\nabla_\theta L(\theta,\lambda)=-\nabla_\theta V^\theta(x^0)+\lambda\nabla_\theta\Lambda^\theta(x^0)\quad\quad\text{and}\quad\quad\nabla_\lambda L(\theta, \lambda)= \Lambda^\theta(x^0)-\alpha.
\end{align*}
Since $\nabla_\theta\Lambda^\theta(x^0)=\nabla_\theta U^\theta(x^0)-2V^\theta(x^0)\nabla_\theta V^\theta(x^0)$, in order to compute $\nabla_\theta\Lambda^\theta(x^0)$ it would be enough to calculate $\nabla_\theta V^\theta(x^0)$ and $\nabla_\theta U^\theta(x^0)$. Using the above definitions, we are now ready to derive the expressions for the gradient of  $V^\theta(x^0)$ and  $U^\theta(x^0)$, which in turn constitute the main ingredients in calculating  $\nabla_\theta L(\theta,\lambda)$.
\begin{lemma}
\label{grad-V-U}
Under (A1) and (A2), we have
\begin{align*}
(1-\gamma)\nabla_\theta V^\theta(x^0)&=\sum_{x,a}\pi^\theta_\gamma(x,a|x^0)\nabla\log\mu(a|x;\theta)Q^\theta(x,a), \\
(1-\gamma^2)\nabla_\theta U^\theta(x^0)&=\sum_{x,a}\widetilde{\pi}^\theta_\gamma(x,a|x^0)\nabla\log\mu(a|x;\theta)W^\theta(x,a) +2\gamma\sum_{x,a,x'}\widetilde{\pi}^\theta_\gamma(x,a|x^0)P(x'|x,a)r(x,a)\nabla_\theta V^\theta(x'),
\end{align*}
where $\widetilde{d}^\theta_\gamma(x|x^0)$ and $\widetilde{\pi}^\theta_\gamma(x,a|x^0)$ are the $\gamma^2$-discounted visiting distributions of state $x$ and state-action pair $(x,a)$ under policy $\mu$, respectively, and are defined as
\begin{align*}
\widetilde{d}^\theta_\gamma(x|x^0)&=(1-\gamma^2)\sum_{n=0}^\infty\gamma^{2n}\Pr(x_n=x|x_0=x^0;\theta), \\
\widetilde{\pi}^\theta_\gamma(x,a|x^0)&=\widetilde{d}^\theta_\gamma(x|x^0)\mu(a|x).
\end{align*}
\end{lemma}
\begin{proof}
The proof of $\nabla V^\theta(x^0)$ is standard and can be found, for instance, in \cite{Peters05NA}. To prove $\nabla U^\theta(x^0)$, we start by the fact that from~\eqref{eq:U-W-Bellman} we have $U(x) = \sum_a\mu(x|a)W(x,a)$. If we take the derivative w.r.t.~$\theta$ from both sides of this equation and obtain

\begin{align}
\label{eq:proof-grad-U1}
\nabla U(x^0)&=\sum_a\nabla\mu(a|x^0)W(x^0,a)+\sum_a\mu(a|x^0)\nabla W(x^0,a) \nonumber \\
&=\sum_a\nabla\mu(a|x^0)W(x^0,a)+\sum_a\mu(a|x^0)\nabla\Big[r(x^0,a)^2+\gamma^2\sum_{x'}P(x'|x^0,a)U(x') \nonumber \\ 
&+2\gamma r(x^0,a)\sum_{x'}P(x'|x^0,a)V(x')\Big] \nonumber \\
&=\underbrace{\sum_a\nabla\mu(a|x^0)W(x^0,a)+2\gamma\sum_{a,x'}\mu(a|x^0)r(x^0,a)P(x'|x^0,a)\nabla V(x')}_{h(x^0)}\nonumber\\
&+\gamma^2\sum_{a,x'}\mu(a|x^0)P(x'|x^0,a)\nabla U(x') \nonumber \\
&=h(x^0)+\gamma^2\sum_{a,x'}\mu(a|x^0)P(x'|x^0,a)\nabla U(x') \\ 
&=h(x^0)+\gamma^2\sum_{a,x'}\mu(a|x^0)P(x'|x^0,a)\nabla\Big[h(x')+\gamma^2\sum_{a',x''}\mu(a'|x')P(x''|x',a')\nabla U(x'')\Big]. \nonumber
\end{align}

By unrolling the last equation using the definition of $\nabla U(x)$ from~\eqref{eq:proof-grad-U1}, we obtain 
\begin{align*}
\nabla U(x^0) &= \sum_{n=0}^\infty\gamma^{2n}\sum_x\Pr(x_n=x|x_0=x^0)h(x)=\frac{1}{1-\gamma^2}\sum_x\widetilde{d}_\gamma(x|x^0)h(x) \nonumber \\
&= \frac{1}{1-\gamma^2}\Big[\sum_{x,a}\widetilde{d}_\gamma(x|x^0)\mu(a|x)\nabla\log\mu(a|x)W(x,a)\\
&+2\gamma\sum_{x,a,x'}\widetilde{d}_\gamma(x|x^0)\mu(a|x)r(x,a)P(x'|x,a)\nabla V(x')\Big] \nonumber \\
&= \frac{1}{1-\gamma^2}\Big[\sum_{x,a}\widetilde{\pi}_\gamma(x,a|x^0)\nabla\log\mu(a|x)W(x,a)\\
&+2\gamma\sum_{x,a,x'}\widetilde{\pi}_\gamma(x,a|x^0)r(x,a)P(x'|x,a)\nabla V(x')\Big]. 
\end{align*}
$\hfill{\blacksquare}$
\end{proof}
In \cite{sutton1999policy}, a policy gradient result analogous to Lemma \ref{grad-V-U} is provided for the value function in the case of full-state representations. In the average reward setting, a similar result helps in extension to incorporate function approximation - see the actor-critic algorithms in \cite{bhatnagar2009natural}\footnote{We extend this to the case of variance-constrained MDP in Section \ref{sec:average-alg}.}. However, a similar approach is not viable for discounted setting and this motivates the use of stochastic optimization techniques like SPSA/SF (cf. \cite{bhatnagar2010actor}).  
The problem is further complicated in the variance-constrained setting that we consider because:
\begin{enumerate}
\item two different sampling distributions, $\pi^\theta_\gamma$ and $\widetilde{\pi}^\theta_\gamma$, are used for $\nabla V^\theta(x^0)$ and $\nabla U^\theta(x^0)$, and 
\item $\nabla V^\theta(x')$ appears in the second sum of $\nabla U^\theta(x^0)$ equation, which implies that we need to estimate the gradient of the value function $V^\theta$ at every state of the MDP, and not just at the initial state $x^0$.
\end{enumerate}
To alleviate the above mentioned problems, we borrow the principle of simultaneous perturbation for estimating the gradient $\nabla_\theta L(\theta,\lambda)$ and develop novel risk-sensitive actor-critic algorithms in the following section. 


\section{Discounted Reward Risk-Sensitive Actor-Critic Algorithms}
\label{sec:discounted-alg}

In this section, we present actor-critic algorithms for optimizing the risk-sensitive measure~\eqref{eq:discounted-risk-measure}. These algorithms are based on two simultaneous perturbation methods: {\em simultaneous perturbation stochastic approximation} (SPSA) and {\em smoothed functional} (SF).


\subsection{Algorithm Structure}
\label{sec:algo-structure}
For the purpose of finding an optimal risk-sensitive policy, a standard procedure would update the policy parameter $\theta$ and Lagrange multiplier $\lambda$ in two nested loops as follows:
\begin{itemize}[$\bullet$]
\item An inner loop that descends in $\theta$ using the gradient of the Lagrangian $L(\theta,\lambda)$ w.r.t. $\theta$, and
\item An outer loop that ascends in $\lambda$ using the gradient of the Lagrangian $L(\theta,\lambda)$ w.r.t. $\lambda$.
\end{itemize}

Using two-timescale stochastic approximation \cite[Chapter 6]{borkar2008stochastic}, the two loops above can run in parallel, as follows:
\begin{align}
\label{eq:theta_descent_det}
\theta_{n+1} &= \Gamma\big[\theta_n - \zeta_2(n) A_n^{-1} \nabla_\theta L(\theta_n,\lambda_n)\big],\\
\lambda_{n+1} &= \Gamma_\lambda\big[\lambda_n + \zeta_1(n) \nabla_\lambda L(\theta_n,\lambda_n)\big],
\end{align}
In the above, 
\begin{itemize}
\item $A_n$ is a positive definite matrix that fixes the order of the algorithm. For the first order methods, $A_n=I$ ($I$ is the identity matrix), while for the second order methods $A_n \rightarrow \nabla^2_\theta L(\theta_n,\lambda_n)$ as $n \rightarrow \infty$.
\item $\Gamma$ is a projection operator that keeps the iterate $\theta_n$ stable by projecting  onto a compact and convex set $\Theta:= \prod_{i=1}^{\kappa_1} [\theta^{(i)}_{\min},\theta^{(i)}_{\max}]$. In particular, for any $\theta \in \R^\kappa_1$,  $\Gamma(\theta) = (\Gamma^{(1)}(\theta^{(1)}),\ldots, \Gamma^{(\kappa_1)}(\theta^{(\kappa_1)}))^T$, with $\Gamma^{(i)}(\theta^{(i)}) := \min(\max(\theta^{(i)}_{\min},\theta^{(i)}),\theta^{(i)}_{\max})$.
\item $\Gamma_\lambda$ is a projection operator that keeps the Lagrange multiplier $\lambda_n$ within the interval $[0,\lambda_{\max}]$, for some large positive constant $\lambda_{\max} < \infty$ and can be defined in an analogous fashion as $\Gamma$.
\item $\zeta_1(n), \zeta_2(n)$ are step-sizes selected such that $\theta$ update is on the faster and $\lambda$ update is on the slower timescale.  Note that another timescale $\zeta_3(n)$ that is the fastest is used for the TD-critic, which provides the estimate of the Lagrangian for a given $(\theta,\lambda)$. 
\end{itemize}

We make the following assumptions on the step-size schedules: \\
  
\noindent
{\bf (A3)} The step size schedules $\{\zeta_3(n)\}$, $\{\zeta_2(n)\}$, and $\{\zeta_1(n)\}$ satisfy
\begin{align}
\label{eq:step1}
&\sum_n \zeta_1(n) = \sum_n \zeta_2(n) = \sum_n \zeta_3(n) = \infty, \\
\label{eq:step2}
&\sum_n \zeta_1(n)^2,\;\;\;\sum_n \zeta_2(n)^2,\;\;\;\sum_n \zeta_3(n)^2<\infty, \\
\label{eq:step3}
&\qquad\qquad\zeta_1(n) = o\big(\zeta_2(n)\big).
\end{align}
Equations~\ref{eq:step1} and~\ref{eq:step2} are standard step-size conditions in stochastic approximation algorithms, and Equation~\ref{eq:step3} ensures that
the policy parameter update is on the faster time-scale $\{\zeta_2(n)\}$, and the Lagrange multiplier update is on the slower time-scale $\{\zeta_1(n)\}$. 

\paragraph{\textbf{Simulation optimization.}}
We operate in a setting where we only observe simulated rewards of the underlying MDP. Thus, it is required to estimate the mean and varaince of the return (we use a TD-critic for this purpose) and then use these estimates to compute gradient of the Lagrangian. 
The gradient $\nabla_\lambda L(\theta,\lambda)$ has a particularly simple form of $(\Lambda^\theta(x^0)-\alpha)$, suggesting the usage of sample variance constraints to perform the dual ascent for Lagrange multiplier $\lambda$. On the other hand, the expression for $\nabla_\theta L(\theta,\lambda)$ is complicated (see Lemma \ref{grad-V-U}) and warrants the usage of a simulation optimization that can provide gradient estimates from sample observation. We employ simultaneous perturbation schemes for estimating the gradient (and in the case of second order methods, the Hessian) of the Lagrangian $L(\theta,\lambda)$. The idea in these methods is to estimate the gradients $\nabla_\theta V^{\theta}(x^0)$ and $\nabla_\theta U^{\theta}(x^0)$ (needed for estimating the gradient $\nabla_\theta L(\theta,\lambda)$) using two simulated trajectories of the system corresponding to policies with parameters $\theta_n$ and $\theta_n^+=\theta_n+p_n$. Here $p_n$ is a perturbation vector that is specific to the algorithm. 

Based on the order, our algorithms can be classified as:
\begin{enumerate}
\item \textbf{First order}: This corresponds to $A_n = I$ in \eqref{eq:theta_descent_det}. The proposed algorithms here include RS-SPSA-G and RS-SF-G, where the former estimates the gradient using SPSA, while the latter uses SF. These algorithms use the following choice for the perturbation vector: $p_n=\beta\Delta_n$. Here $\beta>0$ is a positive constant and $\Delta_n$ is a perturbation random variable, i.e.,~a $\kappa_1$-vector of independent Rademacher (for SPSA) and Gaussian $\N(0,1)$ (for SF) random variables. 
\item \textbf{Second order}: This corresponds to $A_n$ which converges to $\nabla^2 L(\theta_n,\lambda_n)$ as $n\rightarrow \infty$. The proposed algorithms here include RS-SPSA-N and RS-SF-N, where the former uses SPSA for gradient/Hessian estimates and the latter employs SF for the same. These algorithms use the following choice for perturbation vector: For RS-SPSA-N, $p_n=\beta\Delta_n + \beta\widehat\Delta_n$, $\beta>0$ is a positive constant and $\Delta_n$ and $\widehat\Delta_n$ are perturbation parameters that are $\kappa_1$-vectors of independent Rademacher random variables, respectively. For RS-SF-N, $p_n=\beta\Delta_n$, where $\Delta_n$ is a $\kappa_1$ vector of Gaussian $\N(0,1)$ random variables.
\end{enumerate}

\begin{figure}[t]
\centering
\includegraphics[width=4.85in]{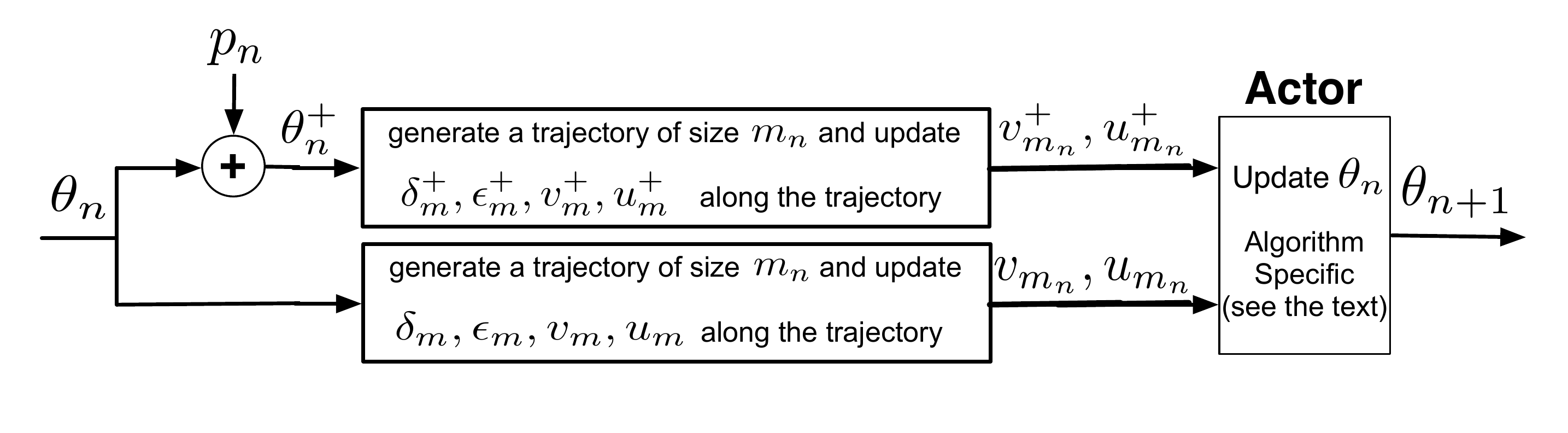}
\caption{The overall flow of our simultaneous perturbation based actor-critic algorithms.}
\label{fig:algorithm-flow}
\end{figure}

\begin{algorithm}
\begin{algorithmic}
\STATE {\bf Input:} parameterized policy $\mu(\cdot|\cdot;\theta)$ and value function feature vectors $\phi_v(\cdot)$ and $\phi_u(\cdot)$
\STATE {\bf Initialization:} policy parameter $\theta=\theta_0$; value function weight vectors $v=v_0$ and $v^+=v^+_0$; square value function weight vectors $u=u_0$ and $u^+=u^+_0$; initial state $x_0\sim P_0(x)$
\FOR{$n = 0,1,2,\ldots$}
\FOR{$m = 0,1,2,\ldots,m_n$}
\STATE Draw action $a_m \sim \mu(\cdot|x_m;\theta_n)$, observe next state $x_{m+1}$ and reward $R(x_m,a_m)$
\STATE Draw action  $a^+_m\sim\mu(\cdot|x^+_m;\theta_n^+)$, observe next state $x^+_{m+1}$ and reward  $R(x^+_m,a^+_m)$
\STATE {\bf Critic Update:} $\;\;$ see~\eqref{eq:critic-discounted} and~\eqref{eq:td-error-discounted} in the text
\ENDFOR  
\STATE {\bf Actor Update:} $\;\;$ Algorithm-Specific
\STATE {\bf Lagrange Multiplier Update:} $\;\;$ see \eqref{eq:lambda-update} in the text
\ENDFOR  
\STATE {\bf return} policy and value function parameters $\theta,\lambda,v,u$
\end{algorithmic}
\caption{Template of the Risk-Sensitive Discounted Reward Actor-Critic Algorithms}\label{algo:discounted-AC}
\end{algorithm}

\noindent
The overall flow of our proposed actor-critic algorithms is illustrated in Figure~\ref{fig:algorithm-flow} and Algorithm \ref{algo:discounted-AC}. The overall operation involves the following two loops: At each time instant $n$, \\

\begin{description}
\item[\textbf{Inner Loop (Critic Update):}] For a fixed policy (given as $\theta_n$), simulate two system trajectories, each of length $m_n$, as follows:\\ 
\begin{inparaenum}[\bfseries1)]
\item {\bf Unperturbed Simulation:} For $m=0,1,\ldots,m_n$, take action $a_m\sim\mu(\cdot|x_m;\theta_n)$, observe the reward $R(x_m,a_m)$, and the next state $x_{m+1}$ in the first trajectory. \\
\item {\bf Perturbed Simulation:} For $m=0,1,\ldots,m_n$,  take action $a^+_m\sim\mu(\cdot|x^+_m;\theta_n^+)$, observe the reward $R(x^+_m,a^+_m)$, and the next state  $x^+_{m+1}$ in the second trajectory. \\
\end{inparaenum}
Using the method of temporal differences (TD)~\cite{Sutton84TC}, estimate the value functions \\$\widehat V^{\theta_n}(x^0)$ and $\widehat V^{\theta_n^+}(x^0)$, and square value functions $\widehat U^{\theta_n}(x^0)$ and $\widehat U^{\theta_n^+}(x^0)$, corresponding to the policy parameter $\theta_n$ and $\theta_n^+$. 
\item[\textbf{Outer Loop (Actor Update):}] Estimate the gradient/Hessian of $\widehat V^{\theta}(x^0)$ and $\widehat U^{\theta}(x^0)$, and hence the gradient/Hessian of Lagrangian $L(\theta,\lambda)$, using either SPSA~\eqref{eq:SPSA-grad} or SF~\eqref{eq:SF-grad} methods. Using these estimates, update the policy parameter $\theta$ in the descent direction using either a gradient or a Newton decrement, and the Lagrange multiplier $\lambda$ in the ascent direction.
\end{description}

\begin{remark}(\textbf{Trajectory length $m_n$}) A simple setting is to have $m_n = C n^{\varsigma}$, where $C$ is a constant and $\varsigma >0$, i.e., have trajectories that increase in length as a function of outer loop index $n$. A constant trajectory length $m_n = \dfrac{C}{1-\gamma}$ is also possible, as after $\dfrac{1}{1-\gamma}$, the discount factor $\gamma$ would have decayed enough to ensure that the value estimate is close enough to the true value.
\end{remark}
In the next section, we describe the TD-critic and subsequently, in Sections \ref{sec:first-algos}--\ref{sec:second-algos}, present the first and second order actor critic algorithms, respectively.

\subsection{TD-Critic}
\label{sec:td}
In our actor-critic algorithms, the critic uses linear approximation for the value and square value functions, i.e.,~$\widehat{V}(x)\approx v\tr\phi_v(x)$ and $\widehat{U}(x)\approx u\tr\phi_u(x)$, where the features $\phi_v(\cdot)$ and $\phi_u(\cdot)$ are from low-dimensional spaces $\R^{\kappa_2}$ and $\R^{\kappa_3}$, respectively.  Let $\Phi_v$ and $\Phi_u$ denote $|\X|\times\kappa_2$ and $|\X|\times\kappa_3$ dimensional matrices, whose $i$th columns are $\phi_v^{(i)}=\big(\phi_v^{(i)}(x),\;x\in\X\big)\tr,\;i=1,\ldots,\kappa_2$ and $\phi_u^{(i)}=\big(\phi_u^{(i)}(x),\;x\in\X\big)\tr,\;i=1,\ldots,\kappa_3$. 
Let $S_v:= \{ \Phi_v v \mid v \in \R^\kappa_2\}$ and $S_u:= \{ \Phi_u u \mid u \in \R^\kappa_3\}$, denote the subspaces within which we approximate the value and square value functions. 
We make the following standard assumption as in~\cite{bhatnagar2009natural}: \\

\noindent
{\bf (A4)} {\em The basis functions $\{\phi_v^{(i)}\}_{i=1}^{\kappa_2}$ and $\{\phi_u^{(i)}\}_{i=1}^{\kappa_3}$ are linearly independent. In particular, $\kappa_2,\kappa_3\ll n$ and $\Phi_v$ and $\Phi_u$ are full rank. Moreover, for every $v\in\R^{\kappa_2}$ and $u\in\R^{\kappa_3}$, $\Phi_vv\neq e$ and $\Phi_uu\neq e$, where $e$ is the $n$-dimensional vector with all entries equal to one.} \\

Let $\Pi_u$ and $\Pi_v$ be operators that project onto $S_v$ and $S_u$, respectively and as a consequence of the above assumption, can be defined as follows:
\begin{align}
 \label{eq:pi-uv}
\Pi_v = \Phi_v (\Phi_v\tr D_\theta \Phi_v)^{-1} \Phi_v\tr D_\theta \text{ and } \Pi_u = \Phi_u (\Phi_u\tr D_\theta \Phi_u)^{-1} \Phi_u\tr D_\theta,
\end{align}
where $D^\theta$ is a diagonal $|\X|\times |\X|$ matrix with entries $d^\theta(x),$ for each $x\in \X$. 

Let $T^\theta = [T_v^\theta; T_u^\theta]$, where $T_v^\theta$ and $T_u^\theta$ denote the Bellman operators for value and square value functions of the policy governed by parameter $\theta$, respectively. These operators are defined as: For any $y\in\R^{2|\X|}$, let $y_v$ and $y_u$ denote the first and last $|\X|$ entries, respectively. Then
\begin{align}
\label{eq:Tv}
T^\theta y &= [ T_v^\theta y; T_u^\theta y], \text{ where }\\
T_v^\theta y&=\boldsymbol{r}^\theta+\gamma\boldsymbol{P}^\theta y_v,\\
T_u^\theta y&=\boldsymbol{R}^\theta\boldsymbol{r}^\theta+2\gamma\boldsymbol{R}^\theta\boldsymbol{P}^\theta y_v+\gamma^2\boldsymbol{P}^\theta y_u, \label{eq:Tu}
\end{align}
 where $\boldsymbol{r}^\theta$ and $\boldsymbol{P}^\theta$ are the reward vector and the transition probability matrix of policy $\theta$, and $\boldsymbol{R}^\theta=diag(\boldsymbol{r}^\theta)$. 

Let 
$\Pi = \left( \begin{array}{cc}                                                                                                                       \Pi_v & 0 \\
0 & \Pi_u     
\end{array} \right)$. Also, for any $y \in \R^{2|\X|}$, define its $\nu$-weighted norm as 
$$\| y \|_\nu = \nu \| y_v \|_{D^\theta} + (1-\nu) \| y_u \|_{D^\theta}.$$ 
We now claim that the projected Bellman operator $\Pi T$ is a contraction mapping w.r.t $\nu$-weighted norm, for any policy $\theta$.
\begin{lemma}
Under (A2) and (A4), there exists a $\nu \in (0,1)$ and $\bar \gamma <1$ such that 
$$\left\| \Pi T y - \Pi T \bar y \right\|_{\nu} \le \bar\gamma \left\| y - \bar y \right\|_{\nu}, \forall y, \bar y \in \R^{2|\X|}.$$
\end{lemma}
\begin{proof}
First, it is well-known that $\Pi_v T_v^\theta$ is a contraction mapping (cf. Lemma 6 in \cite{tsitsiklis1997analysis}). This can be inferred as follows: For any $y, \bar y \in\R^{2|\X|}$,
 $$ \| T_v^\theta y - T_v^\theta \bar y \|_{D^\theta} = \gamma\| y_v - \bar y_v \|_{D^\theta}.$$ 
We have used the fact that $\| P^\theta v \|_{D^\theta} \le \| v \|_{D^\theta}$ for any $v \in \R^{|\X|}$ (For a proof, see Lemma 1 in \cite{tsitsiklis1997analysis}).
The claim that $\Pi_v T_v^\theta$ now follows from the fact that the projection operator $\Pi_v$ is non-expansive.

Now, for any $y, \bar y \in \R^{2|\X|}$, we have
\begin{align}
&\|\Pi_u T_u^\theta y - \Pi_u T_u^\theta \bar y\|_{D^\theta}\nonumber\\
 = & \|2\gamma \Pi_u  R^\theta P^\theta y_v - 2\gamma \Pi_u R^\theta P^\theta \bar y_v + \gamma^2 \Pi_u P^\theta y_u - \gamma^2 \Pi_u P^\theta \bar y_u  \|_{D^\theta}\nonumber\\
 \le & 2\gamma \|\Pi_u  R^\theta P^\theta y_v - \Pi_u R^\theta P^\theta \bar y_v \|_{D^\theta} + \gamma^2 \|y_u -  \bar y_u  \|_{D^\theta}\nonumber\\
 \le & \gamma C_1 \|y_v - \bar y_v \|_{D^\theta} + \gamma^2 \|y_u -  \bar y_u  \|_{D^\theta}.\label{eq:piu}
\end{align}
The first inequality above follows from the aforementioned facts that $P^\theta$ and $\Pi_u$ are non-expansive. The second inequality follows by using equivalence of norms (cf. the justification for Eq. (7) in the proof of Lemma 7 in \cite{Tamar13TD}).

Setting $\nu = \dfrac{\gamma C_1}{\epsilon + \gamma C_1}$, where $\epsilon$ is such that $\gamma + \epsilon < 1$ and plugging in \eqref{eq:piu}, we obtain
\begin{align*}
&\|\Pi T^\theta y - \Pi T^\theta \bar y\|_{\nu}\\
= & \nu \| T_v^\theta y - T_v^\theta \bar y \|_{D^\theta} + (1-\nu) \| \Pi_u T_u^\theta y - \Pi_u T_u^\theta \bar y \|_{D^\theta}\\
 \le & \nu \gamma \|y_v - \bar y_v \|_{D^\theta} +  (1- \nu) \gamma C_1 \|y_v - \bar y_v \|_{D^\theta} + (1-\nu) \gamma^2 \|y_u -  \bar y_u  \|_{D^\theta}\\
 \le & \nu (\gamma + \epsilon) \|y_v - \bar y_v \|_{D^\theta}  + (1-\nu) \gamma \|y_u -  \bar y_u  \|_{D^\theta}\\
 \le &  (\gamma + \epsilon) \|y - \bar y \|_{\nu}.
\end{align*}
The claim follows by setting $\bar \gamma = \gamma + \epsilon$.
$\hfill{\blacksquare}$
\end{proof}

Let $[\Phi_v\bar v;\Phi_u\bar u]$ denote the unique fixed-point of the projected Bellman operator $\Pi T$, i.e., 
\begin{align}
\label{eq:td-fixedpoint}
\Phi_v\bar v = \Pi_v\big(T_v (\Phi_v\bar v)\big), \text{ and } \Phi_u\bar u = \Pi_u\big(T_u(\Phi_u\bar u)\big),
\end{align}
where $\Pi_v$ and $\Pi_u$ project into the linear spaces spanned by the columns of $\Phi_v$ and $\Phi_u$, respectively.

We now describe the TD algorithm that updates the critic parameters corresponding to the value and square value functions (Note that we require critic estimates for both the unperturbed as well as the perturbed policy parameters). This algorithm is an extension of the algorithm proposed by \cite{Tamar13TD} to the discounted setting. Recall from Algorithm \ref{algo:discounted-AC} that, at any instant $n$, the TD-critic runs two $m_n$ length trajectories corresponding to policy parameters $\theta_n$ and $\theta_n + \delta \Delta_n$.

\noindent
{\bf Critic Update:} Calculate the temporal difference (TD)-errors $\delta_m,\delta_m^+$ for the value and $\epsilon_m,\epsilon_m^+$ for the square value functions using~\eqref{eq:td-error-discounted}, and update the critic parameters $v_m,v_m^+$ for the value and $u_m,u_m^+$ for the square value functions as follows:
\begin{align}
\text{\bf Unperturbed:}&\nonumber\\
\label{eq:critic-discounted} v_{m+1}=&v_m + \zeta_3(m) \delta_m \phi_v(x_m),\quad\quad u_{m+1}=u_m + \zeta_3(m) \epsilon_m \phi_u(x_m), \\
\text{\bf Perturbed:}&\nonumber\\
  v^+_{m+1}=&v^+_m + \zeta_3(m) \delta^+_m \phi_v(x^+_m),\quad u^+_{m+1}=u^+_m + \zeta_3(m) \epsilon^+_m \phi_u(x^+_m),
\end{align}
where the TD-errors $\delta_m,\delta_m^+,\epsilon_m,\epsilon_m^+$ in~\eqref{eq:critic-discounted} are computed as
\begin{align}
&\text{\bf Unperturbed:}\nonumber\\
&\delta_m  =  R(x_m, a_m) + \gamma v\tr_m \phi_v(x_{m+1}) - v_m\tr \phi_v(x_m), \label{eq:td-error-discounted}
 \\
&\epsilon_m =  R(x_m, a_m)^2 + 2\gamma R(x_m, a_m)v\tr_m \phi_v(x_{m+1})+\gamma^2 u\tr_m \phi_u(x_{m+1}) - u\tr_m \phi_u(x_m), \nonumber \\
&\text{\bf Perturbed:}\nonumber\\
&\delta^{+}_m  = R(x^+_m, a^+_m) + \gamma v^{+\top}_m \phi_v(x^+_{m+1}) - v^{+\top}_m \phi_v(x^+_m), \label{eq:td-error-discounted-perturb} \\
&\epsilon^+_m = R(x^+_m, a^+_m)^2 + 2\gamma R(x^+_m, a^+_m)v^{+\top}_m \phi_v(x^+_{m+1})+\gamma^2 u^{+\top}_m \phi_u(x^+_{m+1}) - u^{+\top}_m \phi_u(x^+_m). \nonumber
\end{align}
Note that the TD-error $\epsilon$ for the square value function $U$ comes directly from its Bellman equation~\eqref{eq:U-W-Bellman}. Theorem \ref{thm:td} in Section \ref{sec:SPSA-SF-proofs} establishes that the critic parameters $(v_n,u_n)$ governed by \eqref{eq:critic-discounted} converge to the solutions $(\bar v, \bar u)$ of the fixed point equation \eqref{eq:td-fixedpoint}.


\subsection{First-Order Algorithms: RS-SPSA-G and RS-SF-G}
\label{sec:first-algos}

{\bf SPSA}-based estimate for $\nabla V^\theta(x^0)$, and similarly for $\nabla U^\theta(x^0)$, is given by
\begin{align}
 \label{eq:SPSA-grad}
\nabla_i \widehat V^\theta(x^0)\quad\approx\quad \dfrac{\widehat V^{\theta+\beta\Delta}(x^0) - \widehat V^\theta(x^0)}{\beta \Delta^{(i)}},\quad\quad\quad i=1,\ldots,\kappa_1,
\end{align}
where $\Delta$ is a vector of independent Rademacher random variables.
The advantage of this estimator is that it perturbs all directions at the same time (the numerator is identical in all $\kappa_1$ components). So, the number of function measurements needed for this estimator is always two, independent of the dimension $\kappa_1$. However, unlike the SPSA estimates in~\cite{Spall92MS} that use two-sided balanced estimates (simulations with parameters $\theta-\beta\Delta$ and $\theta+\beta\Delta$), our gradient estimates are one-sided (simulations with parameters $\theta$ and $\theta+\beta\Delta$) and resemble those in~\cite{Chen99KW}. The use of one-sided estimates is primarily because the updates of the Lagrangian parameter $\lambda$ require a simulation with the running parameter $\theta$. Using a balanced gradient estimate would therefore come at the cost of an additional simulation (the resulting procedure would then require 
three simulations), which we avoid by using one-sided gradient estimates.

\noindent
{\bf SF}-based method estimates not the gradient of a function $H(\theta)$ itself, but rather the convolution of $\nabla H(\theta)$ with the Gaussian density function $\N(\boldsymbol{0},\beta^2\boldsymbol{I})$, i.e.,
\begin{align*}
C_\beta H(\theta) &= \int \G_\beta(\theta-z)\nabla_z H(z)dz= \int\nabla_z\G_\beta(z)H(\theta-z)dz \\
&= \frac{1}{\beta}\int-z'\G_1(z')H(\theta-\beta z')dz',
\end{align*}
where $\G_\beta$ is a $\kappa_1$-dimensional p.d.f. The first equality above follows by using integration by parts and the second one by using the fact that $\nabla_z\G_\beta(z)=\frac{-z}{\beta^2}\G_\beta(z)$ and by substituting $z'=z/\beta$. As $\beta\rightarrow 0$, it can be seen that $C_\beta H(\theta)$ converges to $\nabla_\theta H(\theta)$ (see Chapter~6 of~\cite{Bhatnagar13SR}). Thus, a one-sided SF estimate of $\nabla V^{\theta}(x^0)$ is given by 
\begin{align}
\label{eq:SF-grad}
\nabla_i\widehat V^\theta(x^0)\quad\approx\quad \frac{\Delta^{(i)}}{\beta} \left(\widehat V^{\theta+\beta\Delta}(x^0) - \widehat V^\theta(x^0)\right),\quad\quad\quad i=1,\ldots,\kappa_1,
\end{align}
where $\Delta$ is a vector of independent Gaussian $\N(0,1)$ random variables. \\

\noindent
{\bf Actor Update:} Estimate the gradients $\nabla V^{\theta}(x^0)$ and $\nabla U^{\theta}(x^0)$ using SPSA~\eqref{eq:SPSA-grad} or SF~\eqref{eq:SF-grad} and update the policy parameter $\theta$ as follows\footnote{By an abuse of notation, we use $v_n$ (resp. $v^+_n, u_n, u^+_n$) to denote the critic parameter $v_{m_n}$ (resp. $v^+_{m_n}, u_{m_n}, u^+_{m_n}$)  obtained at the end of a $m_n$ length trajectory.}: For $i=1,\ldots,\kappa_1$, 
\begin{align}
\text{\bf RS-SPSA-G:} \nonumber \\
\theta_{n+1}^{(i)} &= \Gamma_i\bigg[\theta_n^{(i)} + \frac{\zeta_2(n)}{\beta \Delta_n^{(i)}}\Big(\big(1+2\lambda_n v_n\tr \phi_v(x^0)\big)(v^+_n - v_n)\tr \phi_v(x^0) -\lambda_n(u^+_n - u_n)\tr \phi_u(x^0)\Big)\bigg],\label{eq:actor-spsa-update}\\
\text{\bf RS-SF-G:} \nonumber\\
\theta_{n+1}^{(i)} &= \Gamma_i\bigg[\theta_n^{(i)} + \frac{\zeta_2(n)\Delta_n^{(i)}}{\beta}\Big(\big(1+2\lambda_n v_n\tr \phi_v(x^0)\big)(v^+_n - v_n)\tr \phi_v(x^0) - \lambda_n (u^+_n - u_n)\tr \phi_u(x^0)\Big)\bigg].\label{eq:actor-sf-update}
\end{align}
%
For both SPSA and SF variants, the Lagrange multiplier $\lambda$ is updated as follows:
\begin{align}
\hspace{-3em}\lambda_{n+1} &= \Gamma_\lambda\bigg[\lambda_n + \zeta_1(n)\Big(u\tr_n \phi_u(x^0) - \big(v\tr_n \phi_v(x^0)\big)^2 - \alpha \Big)\bigg].
\label{eq:lambda-update}
\end{align}
In the above, note the following:
\begin{enumerate}[\bfseries1)]
\item $\beta>0$ is a small fixed constant and $\Delta_n^{(i)}$'s are independent Rademacher and Gaussian $\N(0,1)$ random variables in SPSA and SF updates, respectively, 
\item $\Gamma$ and $\Gamma_\lambda$ are projection operators that keep the iterates $(\theta_n,\lambda_n)$ stable and were defined in Section \ref{sec:algo-structure}. These projection operators are necessary to keep the iterates stable and hence, ensure convergence of the algorithms.
\end{enumerate}
We provide a proof of convergence of the first-order SPSA and SF algorithms to a tuple $(\theta^{\lambda^*},\lambda^*)$, which is a (local) saddle point of the risk-sensitive objective function $\widehat L(\theta,\lambda) \stackrel{\triangle}{=} -\widehat{V}^\theta(x^0) + \lambda(\widehat{\Lambda}^\theta(x^0) - \alpha)$. Further, the limit $\theta^{\lambda^*}$ satisfies the variance constraint, i.e.,  $\widehat{\Lambda}^{\theta^{\lambda^*}}(x^0) \le \alpha$. See Theorems \ref{thm:spsa-theta-convergence}--\ref{theorem:lambda} and Proposition \ref{prop:feasible} in Section~\ref{sec:SPSA-SF-proofs} for details. 

\begin{remark}\textbf{(On the bias in gradient estimates)}
 Recall that $\widehat V(\theta)$ is the approximate value function for policy $\theta$. Using a Taylor's expansion of $\widehat V(\cdot)$ around $\theta$, we obtain:
$$ \widehat V(\theta + \beta \Delta) = \widehat V(\theta) + \beta \Delta\tr \nabla_\theta \widehat V(\theta) + \frac{\beta^2}{2} \Delta\tr \nabla_\theta^2 \widehat V(\theta) \Delta +  O(\beta_n^3).$$
Assuming an uniform upper bound $C_2$ on $\nabla^2 \widehat V(\cdot)$ and noting that $\Delta$ are Rademacher, we obtain
\begin{align*}
  \E\left[\left.\left(\dfrac{\widehat V(\theta+\beta \Delta) - \widehat V(\theta)}{\beta \Delta^{(i)}}\right)\right| \theta \right]
= & \E\left[ \dfrac{\Delta\tr \nabla_\theta}{\Delta^{(i)}} \widehat V(\theta)  \left. \right| \theta\right] + O(\beta \kappa_1 C_2)\\
= & \nabla_i \widehat V(\theta) + \E\left[ \sum\limits_{j\ne i} \dfrac{\Delta^{(j)}}{\Delta^{(i)}} \nabla_j \widehat V(\theta) \left. \right| \theta\right] + O( \kappa_1 C_2 \beta )\\
= & \nabla_i \widehat V(\theta) + O(\kappa_1 C_2 \beta ).
\end{align*}
Using similar arguments as above, one can conclude that
\begin{align*}
  \E\left[\left.\left(\dfrac{\widehat U(\theta+\beta \Delta) - \widehat U(\theta)}{\beta \Delta^{(i)}}\right)\right| \theta \right]
= & \nabla_i \widehat U(\theta) + O(\kappa_1 C_3 \beta ),
\end{align*}
where $C_3$ upper bounds $\nabla^2 \widehat U(\cdot)$. From the foregoing along with gradient expression for the Lagrangian and the fact that the value function is upper-bounded since we operate in a finite state-action space, it is easy to infer that the bias of one-sided SPSA estimates of the gradient of the Lagrangian is $O(\beta)$.
Later (in Theorem \ref{thm:spsa-theta-convergence}) we establish that the $\theta$-recursion converges to an $\epsilon$-neighborhood of the set of local minima of the Lagrangian, provided $\beta$ is small enough.
\end{remark}

\begin{remark}
\label{subsec:discount-SR}
\textbf{(Extension to Sharpe Ratio Optimization)}

The gradient of Sharpe ratio (SR), $S(\theta)$, in the discounted setting is given by
\begin{equation*}
\nabla S(\theta)=\frac{1}{\sqrt{\Lambda^\theta(x^0)}}\left(\nabla V^\theta(x^0)-\frac{V^\theta(x^0)}{2\Lambda^\theta(x^0)}\nabla\Lambda^\theta(x^0)\right). 
\end{equation*}
The actor recursions for the variants of the RS-SPSA-G and RS-SF-G algorithms that optimize the SR objective are as follows: \\

{\bf RS-SPSA-G}
\begin{align}
\label{eq:discounted-spsa-Sharpe}
\theta^{(i)}_{n+1}&=\Gamma_i\Bigg(\theta^{(i)}_n+\frac{\zeta_2(n)}{\sqrt{u\tr_n \phi_u(x^0) - \big(v\tr_n \phi_v(x^0)\big)^2} \beta \Delta_n^{(i)}} \bigg((v^+_n - v_n)\tr \phi_v(x^0)\\
&-\frac{v_n\tr \phi_v(x^0)\big((u^+_n - u_n)\tr \phi_u(x^0)-2v\tr_n \phi_v(x^0)(v^+_n - v_n)\tr \phi_v(x^0)\big)}{2\Big(u\tr_n \phi_u(x^0) - \big(v\tr_n \phi_v(x^0)\big)^2\Big)}\bigg)\Bigg). \nonumber
\end{align}

{\bf RS-SF-G}
\begin{align}
\label{eq:discounted-sf-Sharpe}
\theta^{(i)}_{n+1} &= \Gamma_i\Bigg(\theta^{(i)}_n+\frac{\zeta_2(n)\Delta_n^{(i)}}{\beta\sqrt{u\tr_n \phi_u(x^0) - \big(v\tr_n \phi_v(x^0)\big)^2}} \bigg((v^+_n - v_n)\tr \phi_v(x^0)\\
& -\frac{v_n\tr \phi_v(x^0)\big((u^+_n - u_n)\tr \phi_u(x^0)-2v\tr_n \phi_v(x^0)(v^+_n - v_n)\tr \phi_v(x^0)\big)}{2\Big(u\tr_n \phi_u(x^0) - \big(v\tr_n \phi_v(x^0)\big)^2\Big)}\bigg)\Bigg). \nonumber
\end{align}

Note that only the actor recursion changes for SR optimization, while the rest of the updates that include the critic recursions for nominal and perturbed parameters remain the same as before in the SPSA and SF based algorithms. Further, SR optimization does not involve the Lagrange parameter $\lambda$, and thus, the proposed actor-critic algorithms are two time-scale (instead of three time-scale as in the described algorithms) stochastic approximation algorithms in this case.    
\end{remark}

\begin{remark} \textbf{(One-simulation SR variant.)}
For the SR objective, the proposed algorithms can be modified to work with only one simulated trajectory of the system. This is because in the SR case, we do not require the Lagrange multiplier $\lambda$, and thus, the simulated trajectory corresponding to the nominal policy parameter $\theta$ is not necessary. In this implementation, the gradient is estimated as $\nabla_iS(\theta) \approx S(\theta +\beta\Delta)/\beta\Delta^{(i)}$ for SPSA and as $\nabla_iS(\theta) \approx (\Delta^{(i)}/\beta)S(\theta +\beta\Delta)$ for SF.
\end{remark}

\begin{remark} \textbf{(Monte-Carlo Critic)}
In the above algorithms, the critic uses a TD method to evaluate the policies. These algorithms can be implemented with a Monte-Carlo critic that at each time instant $n$ computes a sample average of the total discounted rewards corresponding to the nominal $\theta_n$ and perturbed $\theta_n+\beta\Delta_n$ policy parameter. This implementation would be similar to that in~\citep{tamar2012policy}, except here we use simultaneous perturbation methods to estimate the gradient. 
\end{remark}


\subsection{Second-Order Algorithms: RS-SPSA-N and RS-SF-N}
\label{sec:second-algos}

Recall from Section \ref{sec:algo-structure} that a second-order scheme updates the policy parameter in the following manner:
\begin{align}
\label{eq:second-order-theta}
\theta_{n+1} &= \Gamma\big[\theta_n - \zeta_2(n) \nabla^2_\theta L(\theta,\lambda)^{-1} \nabla_\theta L(\theta,\lambda)\big].
\end{align}
From the above, it is evident that for any second-order method, an estimate of the Hessian $\nabla^2_\theta L(\theta,\lambda)$ of the Lagrangian is necessary, in addition to an estimate of the gradient $\nabla_\theta L(\theta,\lambda)$. As in the case of the gradient based schemes outlined earlier, we employ the simultaneous perturbation technique to develop these estimates. The first algorithm, henceforth referred to as RS-SPSA-N, uses SPSA for the gradient/Hessian estimates. On the other hand, the second algorithm, henceforth referred to as RS-SF-N, uses a smoothed functional (SF) approach for the gradient/Hessian estimates.  As confirmed by our numerical experiments, second order methods are in general more accurate, though at the cost of inverting the Hessian matrix in each step. 


\subsubsection{RS-SPSA-N Algorithm}

The Hessian w.r.t.~$\theta$ of $L(\theta,\lambda)$ can be written as follows:
  
\begin{align}
\label{eq:hd}
\nabla^2_\theta L(\theta,\lambda)&= -\nabla^2_\theta V^\theta(x^0) + \lambda \nabla^2_\theta \Lambda^\theta(x^0)\\ 
&=-\nabla^2 V^\theta(x^0) + \lambda \left(\nabla^2 U^\theta(x^0)-2V^\theta(x^0)\nabla^2 V^\theta(x^0) - 2 \nabla V^\theta(x^0)\nabla V^\theta(x^0)\tr\right). \nonumber 
\end{align}

\noindent
{\bf Critic Update:}
As in the case of the gradient based schemes, we run two simulations. However, perturbed simulation here corresponds to the policy parameter $\theta+\beta(\Delta+\widehat\Delta)$, where $\Delta$ and $\widehat\Delta$ represent vectors of independent $\kappa_1$-dimensional Rademacher random variables. 
The critic parameters $v_n, u_n$ from unperturbed simulation and $v^+_n, u^+_n$ from perturbed simulation are updated as described earlier in Section \ref{sec:td}.\\

\noindent
{\bf Gradient and Hessian Estimates:}
Using an SPSA-based estimation technique (see Chapter 7 of \cite{Bhatnagar13SR}), the gradient and Hessian of the value function $V$, and similarly of the square value function $U$, are estimated as follows: For $i=1,\ldots,\kappa_1,$
\begin{align*}
\nabla_i \widehat V^\theta(x^0)&\quad\approx\quad \dfrac{\widehat V^{\theta+\beta(\Delta+\widehat\Delta)}(x^0) - \widehat V^\theta(x^0)}{\beta \Delta^{(i)}} = \dfrac{(v^+_n-v_n)\tr \phi_v(x^0)}{\beta \Delta^{(i)}}, \\
\nabla^2_{i,j} \widehat V^\theta(x^0)&\quad\approx\quad \dfrac{\widehat V^{\theta+\beta(\Delta+\widehat\Delta)}(x^0) - \widehat V^\theta(x^0)}{\beta^2 \Delta^{(i)}\widehat\Delta^{(j)}} = \dfrac{(v^+_n-v_n)\tr \phi_v(x^0)}{\beta^2 \Delta^{(i)}\widehat\Delta^{(j)}}. \\
\end{align*}
The correctness of the above estimates in the limit as $\beta \rightarrow 0$ can be inferred from Lemma \ref{lemma:spsa-n} in the Appendix. The main idea is to expand using suitable Taylor expansions and observe that the bias terms vanish as $\Delta$, being Rademacher, are zero-mean. As in the case of RS-SPSA, this is an one-sided estimate with the unperturbed simulation required for updating the Lagrange multiplier.

\noindent
{\bf Hessian Update:} Using the critic values from the two simulations, we estimate the Hessian $\nabla^2_\theta L(\theta,\lambda)$ as follows: Let $H_n^{(i,j)}$ denote the $n$th estimate of the $(i,j)$th element of the Hessian. Then, for $i,j=1,\ldots, \kappa_1$, with $i\le j$, the update is
\begin{align}
H^{(i, j)}_{n+1}= H^{(i, j)}_n + \zeta'_2(n)\bigg[&\dfrac{\big(1 + \lambda_n (v_n + v_n^+)\tr\phi_v(x^0)\big)(v_n-v^+_n)\tr \phi_v(x^0)}{\beta^2 \Delta^{(i)}_n\widehat\Delta^{(j)}_n} + \dfrac{\lambda_n (u^+_n-u_n)\tr \phi_u(x^0)}{\beta^2 \Delta^{(i)}_n\widehat\Delta^{(j)}_n} - H^{(i, j)}_n \bigg],\label{eq:hessian-update-spsa}
\end{align}
and for $i > j$, we simply set $H^{(i, j)}_{n+1} = H^{(j, i)}_{n+1}$.  In the above, the step-size  $\zeta'_2(n)$ satisfies
$$\sum_{n} \zeta'_2(n) = \infty; \sum_n {\zeta'_2}^2(n) < \infty,
 \dfrac{\zeta_2(n)}{\zeta'_2(n)}\rightarrow 0 \text{ as } n \rightarrow
\infty.$$ 
The last condition above ensures that the Hessian update proceeds on a faster timescale in comparison to the $\theta$-recursion (see \eqref{eq:actor-spsa-n-update} below).
Finally, we set $H_{n+1} = \Upsilon\big([H^{(i,j)}_{n+1}]_{i,j = 1}^{|\kappa_1|}\big)$, where $\Upsilon(\cdot)$ denotes an operator that projects a square matrix onto the set of symmetric and positive definite matrices. This projection is a standard requirement to ensure convergence of $H_n$ to the Hessian $\nabla^2_\theta L(\theta,\lambda)$ and we state the following standard assumption (cf. \cite[Chapter 7]{Bhatnagar13SR}) on this operator:\\


\noindent
{\bf (A5)} {\em 
For any sequence of matrices  $\{A_n\}$ and $\{B_n\}$ in ${\cal R}^{\kappa_1\times \kappa_1}$
such that ${\displaystyle \lim_{n\rightarrow \infty} \parallel A_n-B_n \parallel}$ $= 0$,
the $\Upsilon$ operator satisfies ${\displaystyle \lim_{n\rightarrow \infty} \parallel \Upsilon(A_n)- \Upsilon(B_n) \parallel}$
$= 0$. Further, for any sequence of matrices $\{C_n\}$ in ${\cal R}^{\kappa_1\times \kappa_1}$, we have
$${\displaystyle \sup_n \parallel C_n\parallel}<\infty \quad \Rightarrow \quad \sup_n \parallel \Upsilon(C_n)\parallel < \infty \text{ and }\sup_n \parallel \{\Upsilon(M_n)\}^{-1} \parallel <\infty.$$}

As suggested in \cite{gill1981practical}, a possible definition of $\Upsilon$ is to perform an eigen-decomposition of $H_n$ and then make all eigenvalues positive. This avoids singularity of $H_n$ and also satisfies the above assumption. In our experiments, we use this scheme for projecting $H_n$.

\noindent
{\bf Actor Update:} Let $M_n \stackrel{\triangle}{=} H_n^{-1}$ denote the inverse of the the Hessian estimate $H_n$. We incorporate a Newton decrement to update the policy parameter $\theta$ as follows:
\begin{align}
\theta_{n+1}^{(i)}= \Gamma_i\bigg[\theta_n^{(i)} &+ \zeta_2(n)\sum\limits_{j = 1}^{\kappa_1} M^{(i, j)}_n\Big(\dfrac{\big(1+2\lambda_n v_n\tr \phi_v(x^0)\big)(v^+_n - v_n)\tr \phi_v(x^0)}{\beta \Delta_n^{(j)}} - \dfrac{\lambda_n(u^+_n - u_n)\tr \phi_u(x^0)}{\beta \Delta_n^{(j)}}\Big)\bigg]. \label{eq:actor-spsa-n-update}
\end{align}
In the long run, $M_n$ converges to $\nabla^2_\theta L(\theta,\lambda)^{-1}$, while the last term in the brackets in~\eqref{eq:actor-spsa-n-update} converges to $\nabla_\theta L(\theta,\lambda)$ and hence, the update~\eqref{eq:actor-spsa-n-update} can be seen to descend in $\theta$ using a Newton decrement. Note that the Lagrange multiplier update here is the same as that in RS-SPSA-G.


\subsubsection{RS-SF-N Algorithm}

\noindent
{\bf Gradient and Hessian Estimates:}
While the gradient estimate here is the same as that in the RS-SF-G algorithm, the Hessian is estimated as follows: Recall that $\Delta = \big(\Delta^{(1)},\ldots,\Delta^{(\kappa_1)}\big)\tr$ is a vector of mutually independent $\N(0,1)$ random variables. Let $\bar{H}(\Delta)$ be a $\kappa_1 \times \kappa_1$ matrix defined as
\begin{equation}
\label{H-bar}
\bar{H}(\Delta) \stackrel{\triangle}{=}
\left[
\begin{array}{cccc}
\big(\Delta^{(1)^2}-1\big) & \Delta^{(1)}\Delta^{(2)} & \cdots &
\Delta^{(1)}\Delta^{(\kappa_1)}\\
\Delta^{(2)}\Delta^{(1)}& \big(\Delta^{(2)^2}-1\big) & \cdots &
\Delta^{(2)}\Delta^{(\kappa_1)}\\
\cdots & \cdots & \cdots & \cdots \\
\Delta^{(\kappa_1)}\Delta^{(1)} & \Delta^{(\kappa_1)}\Delta^{(2)} & \cdots &
\big(\Delta^{(\kappa_1)^2}-1\big)
\end{array}
\right].
\end{equation}
Then, the Hessian $\nabla^2_\theta L(\theta,\lambda)$ is approximated as
\begin{align}
\label{p1}
\nabla^2_\theta L(\theta,\lambda)\approx \frac{1}{\beta^2}  \Big[\bar{H}(\Delta)\big(L(\theta +\beta \Delta,\lambda) - L(\theta,\lambda)\big)\Big].
\end{align}
%
The correctness of the above estimate in the limit as $\beta \rightarrow 0$ can be seen from Lemma \ref{lemma:sf-n} in the Appendix. The main idea involves convolving the Hessian with a Gaussian density function (similar to RS-SF) and then performing integration by parts twice.\\

\noindent
{\bf Critic Update:}
As in the case of the RS-SF-G algorithm, we run two simulations with unperturbed and perturbed policy parameters, respectively. Recall that the perturbed simulation corresponds to the policy parameter $\theta+\beta\Delta$, where $\Delta$ represent a vector of independent $\kappa_1$-dimensional Gaussian $\N(0,1)$ random variables. The critic parameters for both these simulations are updated  as described earlier in Section \ref{sec:td}.\\

\noindent
{\bf Hessian Update:} As in RS-SPSA-N, let $H^{(i,j)}_n$ denote the $(i,j)$th element of the Hessian estimate $H_n$ at time step $t$. Using~\eqref{p1}, we devise the following update rule for the Hessian estimate $H_n$: For $i,j,k=1,\ldots,\kappa_1$, $j< k$, the update is
\begin{align}
H^{(i, i)}_{t + 1}= H^{(i, i)}_n + \zeta'_2(n)\bigg[&\dfrac{\big(\Delta^{(i)^2}_n-1\big)}{\beta^2}\Big(\big(1 + \lambda_n (v_n+ v^+_n)\tr\phi_v(x^0)\big)(v_n-v^+_n)\tr \phi_v(x^0) \nonumber \\
&\qquad\qquad\qquad+ \lambda_n (u^+_n-u_n)\tr \phi_u(x^0)\Big) - H^{(i, i)}_n \bigg],\label{eq:hessian-update-sf1}\\
H^{(j, k)}_{t + 1}= H^{(j, k)}_n + \zeta'_2(n)\bigg[&\dfrac{\Delta^{(j)}_n\Delta^{(k)}_n}{\beta^2}\Big(\big(1 + \lambda_n (v_n+ v^+_n)\tr\phi_v(x^0)\big)(v_n-v^+_n)\tr \phi_v(x^0) \nonumber \\
&\qquad\qquad\qquad+ \lambda_n (u^+_n-u_n)\tr \phi_u(x^0)\Big) - H^{(j,k)}_n \bigg],\label{eq:hessian-update-sf2}
\end{align}
and for $j > k$, we set $H^{(j, k)}_{n+1} = H^{(k, j)}_{n+1}$. The step-size $\zeta'_2(n)$ is as in RS-SPSA-N.  Further, as in the latter algorithm, we set $H_{n+1} = \Upsilon\big([H^{(i,j)}_{n+1}]_{i,j = 1}^{|\kappa_1|}\big)$ and let $M_{n+1} \stackrel{\triangle}{=} H_{n+1}^{-1}$ denote its inverse. \\

\noindent
{\bf Actor Update:} Using the gradient and Hessian estimates from the above, we update the policy parameter $\theta$ as follows:
\begin{align}
\label{eq:actor-sf-n-update}
\theta_{n+1}^{(i)}= \Gamma_i\bigg[\theta_n^{(i)} + \zeta_2(n)\sum\limits_{j = 1}^{\kappa_1} M^{(i, j)}_n\frac{ \Delta_n^{(j)}}{\beta}\Big(&\big(1+2\lambda_n v_n\tr \phi_v(x^0)\big)(v^+_n - v_n)\tr \phi_v(x^0) - \lambda_n(u^+_n - u_n)\tr \phi_u(x^0)\Big)\bigg].
\end{align}
As in the case of RS-SPSA-N, it can be seen that the above update rule is equivalent to descent with a Newton decrement, since $M_n$ converges to $\nabla^2_\theta L(\theta,\lambda)^{-1}$, and the last term in the brackets in~\eqref{eq:actor-sf-n-update} converges to $\nabla_\theta L(\theta,\lambda)$. The Lagrange multiplier $\lambda$ update here is the same as that in RS-SF-G.


\begin{remark}
The second-order variants of the algorithms for SR optimization can be worked out along similar lines as outlined in Section~\ref{sec:second-algos} and the details are omitted here.
\end{remark}



\section{Average Reward Setting}
\label{sec:average-setting}

The average reward under policy $\mu$ is defined as 
\begin{align*}
\rho(\mu) \; = \; \lim_{T\rightarrow\infty}\frac{1}{T}\E\left[\sum_{n=0}^{T-1}R_n\mid \mu\right] \; = \; \sum_{x,a}d^\mu(x)\mu(a|x)r(x,a) \; = \; \sum_{x,a}\pi^\mu(x,a)r(x,a),
\end{align*}
where $d^\mu$ and $\pi^\mu$ are the stationary distributions of policy $\mu$ over states and state-action pairs, respectively (see Section~\ref{sec:preliminaries}).
The goal in the standard (risk-neutral) average reward formulation is to find an {\em average optimal} policy, i.e., $\mu^*=\argmax_\mu\rho(\mu)$.
For all states $x\in\X$ and actions $a\in\A$, the {\em differential} action-value and value functions of policy $\mu$ are defined respectively as
\begin{align*}
Q^\mu(x,a)&=\sum_{n=0}^\infty\E\big[R_n-\rho(\mu)\mid x_0=x,a_0=a,\mu\big], \\
V^\mu(x) &= \sum_a\mu(a|x)Q^\mu(x,a).
\end{align*}
These functions satisfy the following Poisson equations~\cite{puterman1994markov} 
\begin{align}
\label{eq:DifV-Q-Poisson}
\rho(\mu)+V^\mu(x) &= \sum_a\mu(a|x)\big[r(x,a)+\sum_{x'}P(x'|x,a)V^\mu(x')\big], \\
\rho(\mu)+Q^\mu(x,a)&=r(x,a)+\sum_{x'}P(x'|x,a)V^\mu(x').
\end{align}
In the context of risk-sensitive MDPs, different criteria have been proposed to define a measure of {\em variability} in the average reward setting, among which we consider the {\em long-run variance} of $\mu$~\cite{filar1989variance} defined as
\begin{equation}
\label{eq:V}
\Lambda(\mu) \; = \; \sum_{x,a}\pi^\mu(x,a)\big[r(x,a)-\rho(\mu)\big]^2 \; = \; \lim_{T\rightarrow\infty}\frac{1}{T}\E\left[\sum_{n=0}^{T-1}\big(R_n-\rho(\mu)\big)^2\mid \mu\right].
\end{equation}
%
This notion of variability is based on the observation that it is the frequency of occurrence of state-action pairs that determine the variability in the average reward. It is easy to show that 
\begin{equation*}
\Lambda(\mu) = \eta(\mu) - \rho(\mu)^2,\quad\quad \text{where} \quad\quad\eta(\mu)=\sum_{x,a}\pi^\mu(x,a)r(x,a)^2.
\end{equation*}
We consider the following risk-sensitive measure for average reward MDPs in this paper: 
\begin{equation}
\label{eq:average-risk-measure}
\max_\theta\rho(\theta)\quad\quad \text{subject to} \quad\quad \Lambda(\theta)\leq\alpha,
\end{equation}
for a given $\alpha >0$.\footnote{Similar to the discounted setting, the risk-sensitive average reward algorithm proposed in this paper can be easily extended to other risk measures based on the long-term variance of $\mu$, including the Sharpe Ratio (SR), i.e.,~$\max_\theta \rho(\theta)/\sqrt{\Lambda(\theta)}$. The extension to SR will be described in more details in Section~\ref{subsec:average-SR}.} As in the discounted setting, we employ the Lagrangian relaxation procedure to convert~\eqref{eq:average-risk-measure} to the unconstrained problem
\begin{equation*}
\max_\lambda\min_\theta\left(L(\theta,\lambda) \stackrel{\triangle}{=} -\rho(\theta)+\lambda\big(\Lambda(\theta)-\alpha\big)\right).
\end{equation*}
%
As in the discounted setting, we descend in $\theta$ using $\nabla_\theta L(\theta,\lambda)=-\nabla_\theta\rho(\theta)+\lambda\nabla_\theta\Lambda(\theta)$ and ascend in $\lambda$ using $\nabla_\lambda L(\theta, \lambda) = \Lambda(\theta)-\alpha$, to find the saddle point of $L(\theta,\lambda)$. Since $\nabla_\theta\Lambda(\theta)=\nabla_\theta\eta(\theta)-2\rho(\theta)\nabla_\theta\rho(\theta)$, in order to compute $\nabla_\theta\Lambda(\theta)$ it would be enough to calculate $\nabla_\theta\rho(\theta)$ and $\nabla_\theta\eta(\theta)$. Let $U^\mu$ and $W^\mu$ denote the differential value and action-value functions associated with the square reward under policy $\mu$, respectively. These two quantities satisfy the following Poisson equations:
\begin{align}
\label{eq:DifU-W-Poisson}
\eta(\mu)+U^\mu(x) &= \sum_a\mu(a|x)\big[r(x,a)^2+\sum_{x'}P(x'|x,a)U^\mu(x')\big], \nonumber \\
\eta(\mu)+W^\mu(x,a) &= r(x,a)^2+\sum_{x'}P(x'|x,a)U^\mu(x').
\end{align}
The gradients of $\rho(\theta)$ and $\eta(\theta)$ are given by the following lemma:
\begin{lemma}
\label{grad-rho-eta}
Under (A1) and (A2), we have
\begin{align}
\nabla_\theta\rho(\theta)&=\sum_{x,a}\pi^\theta(x,a)\nabla_\theta\log\mu(a|x;\theta)Q(x,a;\theta), \label{eq:grad-rho}\\
\nabla_\theta\eta(\theta)&=\sum_{x,a}\pi^\theta(x,a)\nabla_\theta\log\mu(a|x;\theta)W(x,a;\theta).
\label{eq:grad-eta}
\end{align}
\end{lemma}
\begin{proof}
The proof of $\nabla_\theta\rho(\theta)$ can be found in the literature (e.g.,~\cite{Sutton00PG,Konda00AA}). To prove $\nabla_\theta\eta(\theta)$, we start by the fact that from~\eqref{eq:DifU-W-Poisson}, we have $U(x) = \sum_a\mu(x|a)W(x,a)$. If we take the derivative w.r.t.~$\theta$ from both sides of this equation, we obtain
\begin{align}
\label{eq:proof-grad-eta1}
\nabla U(x) &= \sum_a\nabla\mu(x|a)W(x,a)+\sum_a\mu(x|a)\nabla W(x,a) \nonumber \\
&= \sum_a\nabla\mu(x|a)W(x,a) +\sum_a\mu(x|a)\nabla \big(r(x,a)^2-\eta+\sum_{x'}P(x'|x,a)U(x')\big) \nonumber \\
&= \sum_a\nabla\mu(x|a)W(x,a) - \nabla\eta + \sum_{a,x'}\mu(a|x)P(x'|x,a)\nabla U(x').
\end{align}
The second equality is by replacing $W(x,a)$ from~\eqref{eq:DifU-W-Poisson}. Now if we take the weighted sum, weighted by $d^\mu(x)=d^\theta(x)$, from both sides of~\eqref{eq:proof-grad-eta1}, we have
\begin{align}
\label{eq:proof-grad-eta2}
\sum_xd^\mu(x)\nabla U(x) &= \sum_{x,a}d^\mu(x)\nabla\mu(a|x)W(x,a)-\nabla\eta \nonumber \\ 
&+\sum_{x,a,x'}d^\mu(x)\mu(a|x)P(x'|x,a)\nabla U(x').
\end{align}
The claim follows from the fact that the last sum on the RHS of~\eqref{eq:proof-grad-eta2} is equal to \\$\sum_xd^\mu(x)\nabla U(x)$.
$\hfill{\blacksquare}$
\end{proof}

Note that~\eqref{eq:grad-eta} for calculating $\nabla\eta(\theta)$ has close resemblance to~\eqref{eq:grad-rho} for $\nabla\rho(\theta)$, and thus, similar to what we have for~\eqref{eq:grad-rho}, any function $b:\X\rightarrow\R$ can be added or subtracted to $W(x,a;\theta)$ on the RHS of~\eqref{eq:grad-eta} without changing the result of the integral (see e.g.,~\cite{bhatnagar2009natural}). So, we can replace $W(x,a;\theta)$ with the square reward advantage function $B(x,a;\theta)=W(x,a;\theta)-U(x;\theta)$ on the RHS of~\eqref{eq:grad-eta} in the same manner as we can replace $Q(x,a;\theta)$ with the advantage function $A(x,a;\theta)=Q(x,a;\theta)-V(x;\theta)$ on the RHS of~\eqref{eq:grad-rho} without changing the result of the integral. We define the temporal difference (TD) errors $\delta_n$ and $\epsilon_n$ for the differential value and square value functions as
\begin{align*}
\delta_n &= R(x_n,a_n) - \widehat{\rho}_{n+1} + \widehat{V}(x_{n+1}) - \widehat{V}(x_n), \\
\epsilon_n &= R(x_n,a_n)^2 - \widehat{\eta}_{n+1} + \widehat{U}(x_{n+1}) - \widehat{U}(x_n).
\end{align*}
If $\widehat{V}$, $\widehat{U}$, $\widehat{\rho}$, and $\widehat{\eta}$ are unbiased estimators of $V^\mu$, $U^\mu$, $\rho(\mu)$, and $\eta(\mu)$, respectively, then we show in Lemma~\ref{TD-error-Advantage} that $\delta_n$ and $\epsilon_n$ are unbiased estimates of the advantage functions $A^\mu$ and $B^\mu$, i.e.,~$\E[\left. \delta_n \right|x_n, a_n,\mu] = A^\mu(x_n, a_n)$ and $\E[\left. \epsilon_n \right|x_n, a_n,\mu] = B^\mu(x_n, a_n)$.

\begin{lemma} 
\label{TD-error-Advantage}
For any given policy $\mu$, we have
\begin{align*}
\E[\left. \delta_n \right|x_n, a_n,\mu] = A^\mu(x_n, a_n), \quad\quad\quad\quad\quad \E[\left. \epsilon_n \right|x_n, a_n,\mu] = B^\mu(x_n, a_n).
\end{align*}
\end{lemma}

\begin{proof}
The first statement $\E[\left. \delta_n \right|x_n, a_n,\mu] = A^\mu(x_n, a_n)$ has been proved in Lemma~3 of~\cite{bhatnagar2009natural}, so here we only prove the second statement $\E[\left. \epsilon_n \right|x_n, a_n,\mu] = B^\mu(x_n, a_n)$. we may write
\begin{align*}
\E[\left. \epsilon_n \right|x_n, a_n,\mu] &= \E\big[R(x_n,a_n)^2 - \widehat{\eta}_{n+1} + \widehat{U}(x_{n+1}) - \widehat{U}(x_n)\;|\;x_n, a_n,\mu\big] \\
&= r(x_n,a_n)^2 - \eta(\mu) + \E\big[\widehat{U}(x_{n+1})\;|\;x_n, a_n,\mu\big] - U^\mu(x_n) \\
&= r(x_n,a_n)^2 - \eta(\mu) + \E\Big[\E\big[\widehat{U}(x_{n+1})\;|\;x_{n+1},\mu\big]\;|\;x_n, a_n\Big] - U^\mu(x_n) \\
&= r(x_n,a_n)^2 - \eta(\mu) + \E\big[\widehat{U}(x_{n+1})\;|\;x_n, a_n\big] - U^\mu(x_n) \\
&= \underbrace{r(x_n,a_n)^2 - \eta(\mu) + \sum_{x_{n+1}\in\X}P(x_{n+1}|x_n,a_n)U^\mu(x_{n+1})}_{W^\mu(x,a)} - U^\mu(x_n) \\
&= B^\mu(x,a). 
\end{align*}
$\hfill{\blacksquare}$
\end{proof}

\noindent
From Lemma~\ref{TD-error-Advantage}, we notice that $\delta_n\psi_n$ and $\epsilon_n\psi_n$ are unbiased estimates of $\nabla\rho(\mu)$ and $\nabla\eta(\mu)$, respectively, where $\psi_n=\psi(x_n,a_n)=\nabla\log\mu(a_n|x_n)$ is the {\em compatible} feature (see e.g.,~\cite{Sutton00PG,Peters05NA}). 


\section{Average Reward Risk-Sensitive Actor-Critic Algorithm}
\label{sec:average-alg}

We now present our risk-sensitive actor-critic algorithm for average reward MDPs. Algorithm~\ref{algo:average-AC} presents the complete structure of the algorithm along with the update rules for the average rewards $\widehat{\rho}_n,\widehat{\eta}_n$; TD errors $\delta_n,\epsilon_n$; critic $v_n,u_n$; and actor $\theta_n,\lambda_n$ parameters. The projection operators $\Gamma$ and $\Gamma_\lambda$ are as defined in Section~\ref{sec:discounted-alg}, and similar to the discounted setting, are necessary for the convergence proof of the algorithm. The step-size schedules satisfy (A3) defined in Section \ref{sec:discounted-alg}, plus the step size schedule $\{\zeta_4(n)\}$ satisfies $\zeta_4(n)=k\zeta_3(n)$, for some positive constant $k$. This is to ensure that the average and critic updates are on the (same) fastest time-scale $\{\zeta_4(n)\}$ and $\{\zeta_3(n)\}$, the policy parameter update is on the intermediate time-scale $\{\zeta_2(n)\}$, and the Lagrange multiplier update is on the slowest time-scale $\{\zeta_1(n)\}$. This results in a three time-scale stochastic approximation algorithm. 

\begin{algorithm}
\begin{algorithmic}
\STATE {\bf Input:} parameterized policy $\mu(\cdot|\cdot;\theta)$ and value function feature vectors $\phi_v(\cdot)$ and $\phi_u(\cdot)$
\STATE {\bf Initialization:} policy parameters $\theta=\theta_0$; value function weight vectors $v=v_0$ and $u=u_0$; initial state $x_0\sim P_0(x)$
\FOR{$t = 0,1,2,\ldots$}
\STATE Draw action $a_n \sim \mu(\cdot|x_n;\theta_n)$ and observe the next state $x_{n+1}\sim P(\cdot|x_n,a_n)$ and the reward $R(x_n,a_n)$
\begin{align}
\textrm{\bf Average Updates:} \quad&  \widehat{\rho}_{n+1}=\big(1-\zeta_4(n)\big)\widehat{\rho}_n+\zeta_4(n)R(x_n, a_n),\nonumber\\ &\widehat{\eta}_{n+1}=\big(1-\zeta_4(n)\big)\widehat{\eta}_n +\zeta_4(n)R(x_n, a_n)^2 \nonumber \\
\textrm{{\bf TD Errors:}} \quad& \delta_n=R(x_n, a_n) - \widehat{\rho}_{n+1}+v\tr_n \phi_v(x_{n+1}) - v_n\tr \phi_v(x_n) \nonumber\\ 
&\epsilon_n =R(x_n, a_n)^2-\widehat{\eta}_{n+1}+u\tr_n \phi_u(x_{n+1}) - u\tr_n \phi_u(x_n)\nonumber\\
\textrm{{\bf Critic Update:}}\quad& \label{eq:average-critic-update}
v_{n+1}=v_n + \zeta_3(n) \delta_n \phi_v(x_n), \quad\quad\quad\quad u_{n+1}= u_n + \zeta_3(n) \epsilon_n \phi_u(x_n) \\
\textrm{{\bf Actor Update:}}\quad&  \label{eq:average-actor-update-theta-L}
\theta_{n+1}=\Gamma\Big(\theta_n-\zeta_2(n)\big(-\delta_n\psi_n+\lambda_n(\epsilon_n\psi_n-2\widehat{\rho}_{n+1}\delta_n\psi_n)\big)\Big) \\
\label{eq:average-actor-update-lambda}
&\lambda_{n+1}=\Gamma_\lambda\Big(\lambda_n+\zeta_1(n)(\widehat{\eta}_{n+1}-\widehat{\rho}_{n+1}^2-\alpha)\Big) 
\end{align}
\ENDFOR  
\STATE {\bf return} policy and value function parameters $\theta,\lambda,v,u$
\end{algorithmic}
\caption{Template of the Average Reward Risk-Sensitive Actor-Critic Algorithm}\label{algo:average-AC}
\end{algorithm}

As 
in the discounted setting, the critic uses linear approximation for the differential value and square value functions, i.e.,~$\widehat{V}(x)=v\tr \phi_v(x)$ and $\widehat{U}(x)=u\tr \phi_u(x)$, where $\phi_v(\cdot)$ and $\phi_u(\cdot)$ are feature vectors of size $\kappa_2$ and $\kappa_3$, respectively. 
Although our estimates of $\rho(\theta)$ and $\eta(\theta)$ are unbiased, since we use biased estimates for $V^\theta$ and $U^\theta$ (linear approximations in the critic), our gradient estimates $\nabla_\theta\rho(\theta)$ and $\nabla_\theta\eta(\theta)$, and as a result $\nabla_\theta L(\theta,\lambda)$, are biased. The following lemma shows the bias in our estimate of $\nabla_\theta L(\theta,\lambda)$. 

\begin{lemma} 
\label{bias-average}
The bias of our actor-critic algorithm in estimating $\nabla_\theta L(\theta,\lambda)$ for fixed $\theta$ and $\lambda$ is
\begin{align*}
\B(\theta,\lambda)=\sum_xd^\theta(x)\Big(&-\big(1+2\lambda\rho(\theta)\big)\big[\nabla\bar{V}^\theta(x)-\nabla v^{\theta\top}\phi_v(x)\big] +\lambda\big[\nabla\bar{U}^\theta(x) - \nabla u^{\theta\top}\phi_u(x)\big]\Big),
\end{align*}
where $v^{\theta\top}\phi_v(\cdot)$ and $u^{\theta\top}\phi_u(\cdot)$ are estimates of $V^\theta(\cdot)$ and $U^\theta(\cdot)$ upon convergence of the TD recursion, and 
\begin{align*}
\bar{V}^\theta(x) &= \sum_a\mu(a|x)\big[r(x,a) - \rho(\theta) + \sum_{x'}P(x'|x,a)v^{\theta\top}\phi_v(x')\big], \\
\bar{U}^\theta(x) &= \sum_a\mu(a|x)\big[r(x,a)^2 - \eta(\theta) + \sum_{x'}P(x'|x,a)u^{\theta\top}\phi_u(x')\big].
\end{align*}
\end{lemma}
\begin{proof}
The bias in estimating $\nabla L(\theta,\lambda)$ consists of the bias in estimating $\nabla\rho(\theta)$ and $\nabla\eta(\theta)$. Lemma~4 in~\citet{bhatnagar2009natural} shows the bias in estimating $\nabla\rho(\theta)$ as
\begin{align*}
\E[\delta_n^\theta\psi_n|\theta]=\nabla\rho(\theta)+\sum_{x\in\X}d^\theta(x)\big[\nabla\bar{V}^\theta(x)-\nabla v^{\theta\top}\phi_v(x)\big],
\end{align*}
where $\delta_n^\theta=R(x_n,a_n) - \widehat{\rho}_{n+1} + v^{\theta\top}\phi_v(x_{n+1}) - v^{\theta\top}\phi_v(x_n)$. Similarly we can prove that the bias in estimating $\nabla\eta(\theta)$ is 
\begin{align*}
\E[\epsilon_n^\theta\psi_n|\theta]=\nabla\eta(\theta)+\sum_{x\in\X}d^\theta(x)\big[\nabla\bar{U}^\theta(x)-\nabla u^{\theta\top}\phi_u(x)\big],
\end{align*}
where $\epsilon_n^\theta=R(x_n,a_n) - \widehat{\eta}_{n+1} + u^{\theta\top}\phi_u(x_{n+1}) - u^{\theta\top}\phi_u(x_n)$. The claim follows by putting these two results together and given the fact that $\nabla\Lambda(\theta)=\nabla\eta(\theta)-2\rho(\theta)\nabla\rho(\theta)$ and $\nabla L(\theta,\lambda)=-\nabla\rho(\theta)+\lambda\nabla\Lambda(\theta)$. Note that the following fact holds for the bias in estimating $\nabla\rho(\theta)$ and $\nabla\eta(\theta)$:
\begin{align*}
\sum_xd^\theta(x)\big[\bar{V}^\theta(x) - v^{\theta\top}\phi_v(x)\big]=0, \quad\quad\quad \sum_xd^\theta(x)\big[\bar{U}^\theta(x) - u^{\theta\top}\phi_u(x)\big]=0. \hspace{0.825in} 
\end{align*}
$\hfill{\blacksquare}$
\end{proof}

\begin{remark}\label{subsec:average-SR}
(\textbf{Extension to Sharpe Ratio Optimization})

The gradient of the Sharpe Ratio (SR) in the average setting is given by
\begin{align*}
\nabla S(\theta)&=\frac{1}{\sqrt{\Lambda(\theta)}}\big(\nabla\rho(\theta)-\frac{\rho(\theta)}{2\Lambda(\theta)}\nabla\Lambda(\theta)\big),
\end{align*}
and thus, the actor recursion for the SR-variant of our average reward risk-sensitive actor-critic algorithm is as follows:
\begin{align}
\label{eq:average-actor-update-Sharpe}
\theta_{n+1}&=\Gamma\Big(\theta_n+\frac{\zeta_2(n)}{\sqrt{\widehat{\eta}_{n+1}-\widehat{\rho}_{n+1}^2}}\big(\delta_n\psi_n-\frac{\widehat{\rho}_{n+1}(\epsilon_n\psi_n-2\widehat{\rho}_{n+1}\delta_n\psi_n)}{2(\widehat{\eta}_{n+1}-\widehat{\rho}_{n+1}^2)}\big)\Big).
\end{align}
Note that the rest of the updates, including the average reward, TD errors, and critic recursions are as in the risk-sensitive actor-critic algorithm presented in Algorithm~\ref{algo:average-AC}. Similar to the discounted setting, since there is no Lagrange multiplier in the SR optimization, the resulting actor-critic algorithm is a two time-scale stochastic approximation algorithm.
\end{remark}

\begin{remark}
In the discounted setting, another popular variability measure is the {\em discounted normalized variance}~\citep{filar1989variance}
\begin{equation}
\label{eq:V3}
\Lambda(\mu) = \E\left[\sum_{n=0}^\infty\gamma^n\big(R_n-\rho_\gamma(\mu)\big)^2\right],
\end{equation}
where $\rho_\gamma(\mu)=\sum_{x,a}d^\mu_\gamma(x|x^0)\mu(a|x)r(x,a)$ and $d^\mu_\gamma(x|x^0)$ is the $\gamma$-discounted visiting distribution of state $x$ under policy $\mu$, defined in Section~\ref{sec:preliminaries}. 
%
%
The variability measure~\eqref{eq:V3} has close resemblance to the average reward variability measure~\eqref{eq:V}, and thus, any (discounted) risk measure based on~\eqref{eq:V3} can be optimized similar to the corresponding average reward risk measure~\eqref{eq:V}. 
\end{remark}

\begin{remark}
\textbf{(Simultaneous perturbation analogues)} In the average reward setting, a simultaneous perturbation algorithm would estimate the average reward $\rho$ and the square reward $\eta$ on the faster timescale and use these to estimate the gradient of the performance objective. However, a drawback with this approach, compared to the algorithm proposed above is the necessity for having two simulated trajectories (instead of one) for each policy update.
\end{remark}

In the following section, we establish the convergence of our average reward actor-critic algorithm to a (local) saddle point of the risk-sensitive objective function $L(\theta,\lambda)$. 


\section{Convergence Analysis of the Discounted Reward Risk-Sensitive Actor-Critic Algorithms}
\label{sec:SPSA-SF-proofs}

Our proposed actor-critic algorithms use multi-timescale stochastic approximation and we use the ordinary differential equation (ODE) approach (see Chapter 6 of~\cite{borkar2008stochastic}) to analyze their convergence. We first provide the analysis for the SPSA based first-order algorithm RS-SPSA-G in Section~\ref{subsec:first-proofs} and later provide the necessary modifications to the proof of SF based first-order algorithm and SPSA/SF based second-order algorithms.


\subsection{Convergence of the First-Order Algorithm: RS-SPSA-G}
\label{subsec:first-proofs}

Recall that RS-SPSA-G is a two-loop scheme where the inner loop is a TD critic that evaluates the value/square value functions for both unperturbed as well as perturbed policy parameter. On the other hand, the outer loop is a two-timescale stochastic approximation algorithm, where the faster timescale updates policy parameter $\theta$ in the descent direction using SPSA estimates of the gradient of the Lagrangian and the slower timescale performs dual ascent for the Lagrange multiplier $\lambda$ using sample constraint values. The faster timescale $\theta$-recursion sees the $\lambda$-updates on the slower timescales as quasi-static, while the slower timescale $\lambda$-recursion sees the $\theta$-updates as equilibrated.

 The proof of convergence of the RS-SPSA-G algorithm to a (local) saddle point of the risk-sensitive objective function $\widehat L(\theta,\lambda) \stackrel{\triangle}{=} -\widehat{V}^\theta(x^0) + \lambda(\widehat{\Lambda}^\theta(x^0) - \alpha) {=} -\widehat{V}^\theta(x^0) + \lambda\big(\widehat{U}^\theta(x^0) - \widehat{V}^\theta(x^0)^2  - \alpha\big)$ contains the following three main steps:
\begin{description}
\item[\textbf{Step 1: Critic's Convergence.}] We establish that, for any given values of $\theta$ and $\lambda$ that are updated on slower timescales, the TD critic converges to a fixed point of the projected Bellman operator for value and square value functions.
\item[\textbf{Step 2: Convergence of $\theta$-recursion.}] We utilize the fact that owing to projection, the $\theta$ parameter is stable. Using a Lyapunov argument, we show that the $\theta$-recursion tracks the ODE \eqref{eq:theta-ode} in the asymptotic limit, for any given value of $\lambda$ on the slowest timescale.
\item[\textbf{Step 3: Convergence of $\lambda$-recursion.}] This step is similar to earlier analysis for constrained MDPs . In particular, we show that $\lambda$-recursion in \eqref{eq:actor-spsa-update} converges and the overall convergence of $(\theta_n,\lambda_n)$  is to a local saddle point $(\theta^{\lambda^*},\lambda^*)$ of  $\widehat L(\theta,\lambda)$, with $\theta^{\lambda^*}$  satisfying the variance constraint in \eqref{eq:discounted-risk-measure}.
\end{description} 


\noindent
{\bf Step 1: (Critic's Convergence)} 
Since the critic's update is in the inner loop, we can assume in this analysis that $\theta$ and $\lambda$ are time-invariant quantities. 
The following theorem shows that the TD critic estimates for the value and square value function converge to the fixed point given by \eqref{eq:td-fixedpoint}, for any given policy $\theta$.
\begin{theorem}
\label{thm:td}
Under (A1)-(A4), for any given policy parameter $\theta$ and Lagrange multiplier $\lambda$, the critic parameters $\{v_m\}$ and $\{u_m\}$ governed by the recursions of \eqref{eq:critic-discounted} converge almost surely, i.e., 
$$\text{As } m\rightarrow \infty, v_m \rightarrow \bar v \text{ and }u_m \rightarrow \bar u \text{ a.s.} $$ 
In the above $\bar v$ and $\bar u$ are the solutions to the TD fixed point equations for policy $\theta$ (see \eqref{eq:td-fixedpoint} in Section \ref{sec:td}.%
\end{theorem}

\begin{remark}
 It is easy to conclude from the above theorem that the TD critic parameters for the perturbed policy parameter also converge almost surely, i.e.,  
$v^+_m \rightarrow \bar v^+$ and $u^+_m \rightarrow \bar u^+$ a.s., where $\bar v^+$ and $\bar u^+$ are the unique solutions to TD fixed point relations for perturbed policy $\theta + \delta \Delta$. Here $\Delta$ is a fixed realization of the perturbation random variable that is updated on the outer loop.
\end{remark}

\begin{proof}
%
The $v$-recursion in \eqref{eq:critic-discounted} is performing TD) with function approximation for the value function,  while the $u$-recursion is doing the same for the square value function.  The convergence of $v$-recursion to the fixed point in \eqref{eq:td-fixedpoint} can be inferred from \cite{tsitsiklis1997analysis}. 

Using an approach similar to \cite{tamar2013policy}, we club both $v$ and $u$ recursions and establish convergence using a stability argument in the following:
Let $w_m = (v_m, u_m)\tr$. Then, \eqref{eq:critic-discounted} can be seen to be equivalent to
\begin{align}
\label{eq:critic-equiv}
 w_{m+1} =& w_m + \zeta_3(m) ( M w_m + \xi + \Delta M_{m+1}), \text{ where }\\
M =& \left(\begin{array}{cc}
                  \Phi_v\tr D^\theta(\gamma P^\theta-I)\Phi_v & 0\\
		  2\gamma\Phi_u\tr D^\theta R^\theta P^\theta \Phi_v & \Phi_u\tr D^\theta(\gamma^2 P^\theta-I)\Phi_u
                 \end{array}
 \right) \text{ and }\nonumber\\
\xi =& \left(\begin{array}{c}
         \Phi_v\tr D^\theta r^\theta\\
	 \Phi_u\tr D^\theta R^\theta r^\theta
        \end{array}\right). \nonumber
\end{align}
Further, $\Delta M_{m+1}$ is a martingale difference, i.e., $\E[\Delta M_{m+1} \mid \F_m] = 0$, where $\F_m$ is the sigma field generated by $w_l, \Delta M_l, l\le m$. 


Let $h(w)=Mw+\xi$. Then, the ODE associated with \eqref{eq:critic-equiv} is
\begin{align}
 \dot w_t = h(w_t).
\end{align}
The above ODE has a unique globally asymptotically stable equilibrium, since $M$ is a negative definite. To see the latter fact, observe that $M$ is block triangular and hence its eigenvalues are that of $\Phi_v\tr D^\theta(\gamma P^\theta-I)\Phi_v$ and $\Phi_u\tr D^\theta(\gamma^2 P^\theta-I)\Phi_u$. It can be inferred from Theorem 2 of \cite{tsitsiklis1997analysis} that the aforementioned matrices are negative definite. For the sake of completeness, we provide a brief sketch in the following:
For any $V \in \R^{|\X|}$, it can be shown that $\left\| P^\theta V\right\|_{D^\theta} \le \left\| V\right\|_{D^\theta}$ (see Lemma 1 in \cite{tsitsiklis1997analysis} for a proof).
Now, 
\begin{align*}
V\tr D^\theta \gamma P^\theta V \le& \gamma \left\|(D^\theta)^{1/2} V \right\| \left|(D^\theta)^{1/2} PV \right\| \\
 =& \gamma \left\|V \right\|_{D^\theta} \left| PV \right\|_{D^\theta} \\
  \le& \gamma \left\|V \right\|^2_{D^\theta}.
\end{align*}
Hence, $V\tr D^\theta (\gamma P^\theta-I) V \le (\gamma - 1)\left\|V \right\|^2_{D^\theta} <0$. By (A4), we know that $\Phi_v$ is full rank implying the negative definiteness of $\Phi_v\tr D^\theta(\gamma P^\theta-I)\Phi_v$. Using the same argument as above and replacing $\Phi_v$ with $\Phi_u$ and $\gamma$ with $\gamma^2$,  one can conclude that $\Phi_u\tr D^\theta(\gamma^2 P^\theta-I)\Phi_u$.

The final claim now follows by applying Theorems 2.1-2.2(i) of \cite{borkar2000ode}, provided we verify assumptions (A1)-(A2) there. The latter assumptions are given as follows:

\noindent \textbf{(A1)} The function $h$ is Lipschitz. For any $c$, define $h_c(w) = h(cw)/c$. Then, there exists a continuous function $h_\infty$ such that $h_c \rightarrow h_\infty$ as $c \rightarrow \infty$ uniformly on compacts. Furthermore, origin is an asymptotically stable equilibrium for the ODE 
\begin{align}
\label{eq:hinfty}
\dot w_t= h(w_t).
\end{align}

\noindent \textbf{(A2)} The martingale difference $\{\Delta M_{m}, m\ge 1\}$ is square-integrable with 
$$\E [\left\| \Delta M_{m+1} \right\|^2 \mid \F_m] \le C_0 (1 + \left\| w_m \right\|^2), m\ge 0,$$
where $C_0 < \infty$.

It is straightforward to verify (A1), as $h_c(w) = Mw + \xi/c$ converges to $h_\infty(w) = Mw$ as $c\rightarrow \infty$. Given that $M$ is negative definite, it is easy to see that origin is a asymptotically stable equilibrium for the ODE \eqref{eq:hinfty}.
(A2) can also be verified by using the same arguments that were used to show that the martingale difference associated with the regular TD algorithm with function approximation satisfies a bound on the second moment (cf. \cite{tsitsiklis1997analysis}).
$\hfill{\blacksquare}$
\end{proof} 


\noindent
{\bf Step 2: (Analysis of $\theta$-recursion)} 
Since $m_n \rightarrow \infty$ as $n \rightarrow \infty$, we can assume that the inner TD critic loop has converged for the purpose of analysing the $\theta$-recursion in \eqref{eq:actor-spsa-update}. 
Due to timescale separation, the value of $\lambda$ (updated on a slower timescale) is assumed to be constant for the analysis of the $\theta$-update. 
 To see this in rigorous terms, first rewrite the $\lambda$-recursion as
 $$
\lambda_{n+1} = \Gamma_\lambda\bigg[\lambda_n + \zeta_2(n) \hat H(n)\bigg].$$
 where $\hat H(n) = \frac{\zeta_1(n)}{\zeta_2(n)} \Big(u\tr_n \phi_u(x^0) - \big(v\tr_n \phi_v(x^0)\big)^2 - \alpha \Big)$. Since the critic recursions converge, it is easy to see that $\sup_n \hat H(n)$ is finite. Combining with the observation that $\frac{\zeta_1(n)}{\zeta_2(n)} = o(1)$ due to the assumption (A3) on step-sizes, we see that the $\lambda$-recursion above tracks the ODE $\dot \lambda = 0$.

In the following, we show that the update of $\theta$ is equivalent to gradient descent for the function $\widehat L(\theta,\lambda)$ and converges to a limiting set that depends on $\lambda$. 

Consider the following ODE
\begin{align}
\label{eq:theta-ode}
\dot{\theta}_t = \check{\Gamma}\left ( \nabla_\theta \widehat L(\theta_t, \lambda)\right),
\end{align}
with the limiting set $\Z_\lambda=\big\{\theta\in C:\check\Gamma\big(\nabla \widehat L(\theta_t,\lambda)\big)=0\big\}$. 
In the above, $\check{\Gamma}(\cdot)$ is a projection operator  that ensures the evolution of $\theta$ via the ODE (\ref{eq:theta-ode}) stays within the set $\Theta:= \prod_{i=1}^{\kappa_1} [\theta^{(i)}_{\min},\theta^{(i)}_{\max}]$ and is defined as follows: For any bounded continuous function $f(\cdot)$,
\begin{align}
\label{eq:Pi-bar-operator}
\check{\Gamma}\big(f(\theta)\big) = \lim\limits_{\tau \rightarrow 0}
\dfrac{\Gamma\big(\theta + \tau f(\theta)\big) - \theta}{\tau}.
\end{align}
 Notice that the limit above may not exist and in that case, as pointed out on pp. 191 of \cite{kushner-clark}, one can define $\check{\Gamma}(f(\theta))$ to be the set of all possible limit points. 
From the definition above, it can be inferred that for $\theta$ in the interior of $\Theta$, $\check{\Gamma}(f(\theta)) = f(\theta)$, while for $\theta$ on the boundary of $\Theta$, $\check{\Gamma}(f(\theta))$ is the projection of $f(\theta)$ onto the tangent space of the boundary of $\Theta$ at $\theta$.

The main result regarding the convergence of the policy parameter $\theta$ for both the RS-SPSA-G and RS-SF-G algorithms is as follows:
\begin{theorem}
\label{thm:spsa-theta-convergence}
Under (A1)-(A4), for any given Lagrange multiplier $\lambda$ and $\varepsilon > 0$, there exists $\beta_0 >0$ such that for all $\beta \in (0, \beta_0)$, $\theta_n \rightarrow \theta^* \in \Z^\varepsilon_{\lambda}$ almost surely. Here $\Z_\lambda^\varepsilon=\big\{\theta\in C:||\theta-\theta_0||<\varepsilon,\theta_0\in \Z_\lambda\big\}$ denotes the set of points in the $\varepsilon$-neighborhood of $\Z_\lambda$.  
\end{theorem}

In order to the prove the above claim, we require the well-known Hirsch lemma (see \cite[pp. 339]{hirsch}). For the sake of completeness, we recall this result below.

Consider the ODE:
\begin{align}
\label{appendix:ode1}
\dot{\theta}_t = h(\theta_t). 
\end{align}
Let $K$ be an asymptotically stable attractor for the above ODE and let $K^\epsilon$ denote its $\epsilon$-neighbourhood.
Given $T$, $\eta >0$, we call a bounded, measurable 
$y(\cdot): \R^{+}\cup \{0\}
\rightarrow \R^{N}$, a $(T,\eta)$-perturbation of (\ref{appendix:ode1}) if
there exist $0=T_0 <T_1 <T_2 <\cdots <T_r \uparrow \infty$ with
$T_{r+1}-T_r \geq T$ $\forall r$ and solutions $\theta^r(t)$, $t\in [T_r,T_{r+1}]$
of (\ref{appendix:ode1}) for $r \geq 0$, such that
\[ \sup_{t\in [T_r, T_{r+1}]} \parallel \theta^r(t) - y(t)\parallel <\eta.\]

\begin{lemma}[Hirsch Lemma]
\label{hirsch-lemma}
Given $\epsilon$, $T >0$, $\exists \bar{\eta} >0$ such that for
all $\Delta \in (0, \bar{\eta})$, every $(T,\eta)$-perturbation of (\ref{appendix:ode1})
converges to $K^\epsilon$. 
\end{lemma}

\begin{proof} (\textbf{Theorem \ref{thm:spsa-theta-convergence}})
The $\theta$-update in~\eqref{eq:actor-spsa-update} can be rewritten using the converged TD-parameters $(\bar v, \bar u)$ and $(\bar v^+, \bar u^+)$ as 

\begin{align}
\theta_{n+1}^{(i)} = &\Gamma_i \bigg( \theta_n^{(i)} + \zeta_2(n) \Big(\big(1+2\lambda \bar{v}\tr \phi_v(x^0)\big)\dfrac{(\bar v^+ - \bar v)\tr \phi_v(x^0)}{\beta \Delta_n^{(i)}}- \lambda\dfrac{(\bar u^+ - \bar u)\tr \phi_u(x^0)}{\beta \Delta_n^{(i)}} + \xi_{1,n} \Big)\bigg)\label{eq:theta-equiv},
\end{align}

\noindent
where 
\begin{align*}
\xi_{1,n}:=&  \bigg(\big(1+2\lambda v_n\tr \phi_v(x^0)\big)\dfrac{( v_n^+ - v_n)\tr \phi_v(x^0)}{\beta \Delta_n^{(i)}}- \lambda\dfrac{( u_n^+ - u_n)\tr \phi_u(x^0)}{\beta \Delta_n^{(i)}}\bigg)\\
& - \bigg(\big(1+2\lambda \bar{v}\tr \phi_v(x^0)\big)\dfrac{(\bar v^+ - \bar v)\tr \phi_v(x^0)}{\beta \Delta_n^{(i)}}- \lambda\dfrac{(\bar u^+ - \bar u)\tr \phi_u(x^0)}{\beta \Delta_n^{(i)}}\bigg).
\end{align*}
Since the trajectory length $m_n \rightarrow \infty$ as $n\rightarrow \infty$, the TD-critic converges in the inner loop (see Theorem \ref{thm:td}) and hence, $\xi_{1,n} = o(1)$. Thus, $\xi_{1,n}$ term can be ignored in the asymptotic analysis of $\theta$-recursion.

\noindent
Recall that $\bar v^+$ and $\bar v$ are converged critic parameters corresponding to policies $\theta$ and $\theta + \beta \Delta$. Letting
$\widehat V(\theta) = \bar v\tr \phi_v(x^0)$, we obtain\footnote{The conditional expectation is taken with respect to the common distribution of the perturbations $\Delta^{(i)}$.}
\begin{align*}
 \E\left[\dfrac{(\bar v^+ - \bar v)\tr\phi_v(x^0)}{\beta \Delta^{(i)}} \left. \right| \theta,\lambda\right] 
=& \E\left[\dfrac{\widehat V(\theta) - \widehat V(\theta + \beta \Delta_n)}{\beta \Delta^{(i)}} \left. \right| \theta,\lambda\right]\\
= & \E\left[ \beta \Delta_n\tr \nabla_\theta \widehat V(\theta) + \xi(\beta) \left. \right| \theta,\lambda\right]\\
= & \nabla_i \widehat V(\theta) + \E\left[ \sum\limits_{j\ne i} \dfrac{\Delta^{(j)}}{\Delta^{(i)}} \nabla_j \widehat V(\theta) \left. \right| \theta,\lambda\right] + \xi(\beta)\\
\rightarrow & \nabla_i \widehat V(\theta) \text{ as } \beta \rightarrow 0.
\end{align*}
The second equality above follows by expanding using Taylor's expansion of $\hat V(\cdot)$ around $\theta$, whereas the third equality follows by using the fact that $\Delta^{(i)}_n$'s are independent Rademacher random variables. Note that $\xi(\beta)$ in the second equality above can be seen to converge to zero as $\beta \rightarrow 0$.

On similar lines, letting $\widehat U(\theta) = \bar u\tr \phi_u(x^0)$, it can be seen that
\begin{align*}
\E\left[ \dfrac{(\bar u^+ - \bar u)\tr \phi_u(x^0)}{\beta \Delta^{(i)}} \left. \right| \theta,\lambda\right] 
\stackrel{\beta \rightarrow 0}{\longrightarrow} \nabla_i \hat U(\theta).
\end{align*}
Plugging the above in \eqref{eq:theta-equiv}, we obtain
\begin{align*}
\theta_{n+1}^{(i)} = &\Gamma_i \bigg( \theta_n^{(i)} + \zeta_2(n) \Big(\big(1+2\lambda \hat V(\theta_n)\big)\nabla_i \hat V(\theta_n)- \lambda \nabla_i U(\theta_n)  \Big)\bigg),\\
= & \Gamma_i \bigg( \theta_n^{(i)} + \zeta_2(n) \Big(\nabla_i \hat L(\theta_n,\lambda) \Big)\bigg),
\end{align*}
as $\beta \rightarrow 0$.

Thus,~\eqref{eq:actor-spsa-update} can be seen to be a discretization of the ODE~\eqref{eq:theta-ode}. Further, $\Z_\lambda$ is an asymptotically stable attractor for the ODE~\eqref{eq:theta-ode}, with $\widehat L(\theta,\lambda)$ itself serving as a strict Lyapunov function. This can be inferred as follows:
\begin{align*}
\dfrac{d \widehat L(\theta,\lambda)}{d t}  
= \nabla_\theta \widehat L(\theta,\lambda) \dot \theta
= \nabla_\theta \widehat L(\theta,\lambda) \check\Gamma\big(-\nabla_\theta \widehat L(\theta,\lambda)\big) < 0.  
\end{align*}
Define a linear interpolated trajectory for the $\theta$-recursion in \eqref{eq:actor-spsa-update} as follows:
Let $s(n) = \sum_{i=0}^{n-1} \zeta_2(i)$. $\bar \theta_t$ is a piecewise linear interpolation defined according to
$\bar \theta_{t(n)} = \theta_n$ with linear interpolation on $[s(n), s(n+1)].$
Now, using standard stochastic approximation arguments (cf. \cite[Theorem 5.12]{Bhatnagar13SR}), $\bar \theta_t$ can be seen to be a $(T,\eta)$-perturbation of the ODE \eqref{eq:theta-ode}.
The claim now follows from Hirsch lemma. 
$\hfill{\blacksquare}$
\end{proof}

\noindent
{\bf Step 3: (Analysis of $\lambda$-recursion and Convergence to a Local Saddle Point)} We first show that the $\lambda$-recursion converges and then prove that the whole algorithm converges to a local saddle point of $\widehat L(\theta,\lambda)$. 

We define the following ODE governing the evolution of $\lambda$: 
\begin{align}
\label{eq:lambda-ode}
 \dot \lambda_t \;\;=\;\; \check\Gamma_\lambda\big[\widehat \Lambda^{\theta^{\lambda_t}}(x^0) - \alpha\big] \;\;=\;\; \check\Gamma_\lambda\big[\widehat U^{\theta^{\lambda_t}}(x^0) - \widehat V^{\theta^{\lambda_t}}(x^0)^2 - \alpha\big],
\end{align}
where $\theta^{\lambda_t}$ is the limiting point of the $\theta$-recursion corresponding to ${\lambda_t}$. Further, $\check{\Gamma}_\lambda$ is
an operator similar to the operator $\check{\Gamma}$ defined in \eqref{eq:Pi-bar-operator} and is defined as follows: For any bounded continuous function $f(\cdot)$,
\begin{align}
\label{eq:Pi-bar-operator-lambda}
\check{\Gamma}_\lambda\big(f(\lambda)\big) = \lim\limits_{\tau \rightarrow 0}
\dfrac{\Gamma_\lambda\big(\lambda + \tau f(\lambda)\big) - \lambda}{\tau}.
\end{align}

%
%
%
\begin{theorem}
\label{theorem:lambda}
$\lambda_n \rightarrow \F$ almost surely as $n \rightarrow \infty$, where $\F \stackrel{\triangle}{=}\big\{\lambda\mid \lambda \in [0,\lambda_{\max}],\;\check\Gamma_\lambda\big[\widehat \Lambda^{\theta^\lambda}(x^0)-\alpha\big]=0,\;\theta^\lambda \in \Z_\lambda \big\}$.
\end{theorem}

\begin{proof}
The proof follows using standard stochastic approximation arguments. The first step is to rewrite the $\lambda$-recursion as follows:
\begin{align*}
\lambda_{n+1} &= \Gamma_\lambda\bigg[\lambda_n + \zeta_1(n)\Big(\bar u\tr \phi_u(x^0) - \big(\bar v\tr \phi_v(x^0)\big)^2 - \alpha + \xi_{2,n} \Big)\bigg],
\end{align*}
where $\xi_{2,n}:= \Big(u_n\tr \phi_u(x^0) - \big( v_n\tr \phi_v(x^0)\big)^2\Big) - \Big(\bar u\tr \phi_u(x^0) - \big(\bar v\tr \phi_v(x^0)\big)^2 \Big)$. Note that the converged critic parameters $\bar v$ and $\bar u$ are for the policy $\theta^{\lambda_n}$. The latter is a limiting point of the $\theta$-recursion, with the Lagrange multiplier $\lambda_n$. Owing to convergence of $\theta$-recursion and also TD-critic in the inner loop, we can conclude that $\xi_{2,n} = o(1)$. Thus, $\xi_{2,n}$ adds an asymptotically vanishing bias term to the $\lambda$-recursion above.
The claim follows by applying the standard result in Theorem 2 of \cite{borkar2008stochastic} for convergence of stochastic approximation schemes. 
$\hfill{\blacksquare}$
\end{proof} 

Recall that $\widehat L(\theta,\lambda) \stackrel{\triangle}{=} -\widehat{V}^\theta(x^0) + \lambda(\widehat{\Lambda}^\theta(x^0) - \alpha)$ and hence
 $\nabla_\lambda \widehat L(\theta,\lambda) = \widehat{\Lambda}^\theta(x^0) - \alpha$. Thus, 
$$ \check\Gamma_\lambda\big[\widehat \Lambda^{\theta^\lambda}(x^0)-\alpha\big]=0,$$
is the same as
$$\check\Gamma_\lambda \nabla_\lambda \widehat L(\theta^\lambda,\lambda) = 0.$$
As in \cite{borkar2005actor}, we invoke the envelope theorem of mathematical economics~\citep{mas1995microeconomic} to conclude that the ODE \eqref{eq:lambda-ode} is equivalent to the following 
\begin{align}
\label{eq:lambda-ode-equiv}
\dot \lambda_t = \check\Gamma_\lambda\big[\nabla_\lambda \widehat L(\theta^{\lambda_t},\lambda_t)\big]. 
\end{align}
Note that the above has to interpreted in the {\em Cartheodory} sense, i.e., as the following integral equation
$$ \lambda_t = \lambda_0 + \int_0^t \check\Gamma_\lambda\big[\nabla_\lambda \widehat L(\theta^{\lambda_s},\lambda_s)\big] ds.$$  
As noted in Lemma 4.3 of \cite{borkar2005actor}, using the generalized envelope theorem from \cite{milgrom2002envelope} it can be shown that the RHS of \eqref{eq:lambda-ode-equiv} coincides with that of \eqref{eq:lambda-ode} at differentiable points, while the ODE spends zero time at non-differentiable points (except at the points of maxima).

We next claim that the limit $\theta^{\lambda^*}$ corresponding to $\lambda^*$ satisfies the variance constraint in \eqref{eq:discounted-risk-measure}, i.e.,
\begin{proposition}
\label{prop:feasible}
 For any $\lambda^*$ in $\hat\F \stackrel{\triangle}{=}\big\{\lambda\mid \lambda \in [0,\lambda_{\max}),\;\check\Gamma_\lambda\big[\widehat \Lambda^{\theta^\lambda}(x^0)-\alpha\big]=0,\;\theta^\lambda \in \Z_\lambda \big\}$, the corresponding limiting point $\theta^{\lambda^*}$ satisfies the variance constraint $\widehat{\Lambda}^{\theta^{\lambda^*}}(x^0) \le \alpha$.
\end{proposition}
\begin{proof}
 Follows in a similar manner as Proposition 10.6 in \cite{Bhatnagar13SR}.
\end{proof}

From Theorems \ref{thm:spsa-theta-convergence}--\ref{theorem:lambda} and Proposition \ref{prop:feasible}, it is evident that the actor recursion \eqref{eq:actor-spsa-update} converges to a tuple $(\theta^{\lambda^*},\lambda^*)$ that is a local minimum w.r.t.~$\theta$ and a local maximum w.r.t.~$\lambda$ of $\widehat L(\theta,\lambda)$. In other words, overall convergence is to a (local) saddle point of $\widehat L(\theta,\lambda)$. Further, the limit is also feasible for the constrained problem in \eqref{eq:discounted-risk-measure} as $\theta^{\lambda^*}$ satisfies the variance constraint there.

\subsection{Convergence of the First-Order Algorithm: RS-SF-G}
Note that since RS-SPSA-G and RS-SF-G use different methods to estimate the gradient, their proofs only differ in the second step, i.e.,~the convergence of the policy parameter $\theta$. 
\subsection*{\textbf{Proof of Theorem~\ref{thm:spsa-theta-convergence} for SF}}
\begin{proof}

As in the case of the SPSA algorithm, we rewrite the $\theta$-update in~\eqref{eq:actor-sf-update} using the converged TD-parameters and constant $\lambda$ as 
\begin{align*}
\theta_{n+1}^{(i)} = \Gamma_i\bigg(\theta_n^{(i)} &- \zeta_2(n)\Big(\frac{-\Delta_n^{(i)}\big(1+2\lambda \bar{v}\tr \phi_v(x^0)\big)}{\beta}(\bar v^+ - \bar v)\tr \phi_v(x^0) + \dfrac{\lambda\Delta^{(i)}_n}{\beta}(\bar u^+ - \bar u)\tr \phi_u(x^0) + \xi_{1,n}\Big)\bigg),
\end{align*}
where $\xi_{1,n} \rightarrow 0$ (convergence of TD in the critic and as a result convergence of the critic's parameters to $\bar v, \bar u,\bar v^+, \bar u^+$) in lieu of Theorem \ref{thm:td}. Next, we establish that \\$\E\left[ \dfrac{\Delta^{(i)}}{\beta}(\bar v^+ - \bar v)\tr \phi_v(x^0) \left. \right| \theta,\lambda\right]$ is an asymptotically correct estimate of the gradient of $\widehat V(\theta)$ in the following:
\begin{align*}
\E\left[ \dfrac{\Delta^{(i)}}{\beta}(\bar v^+ - \bar v)\tr \phi_v(x^0) \left. \right| \theta,\lambda\right]\stackrel{\beta \rightarrow 0}{\longrightarrow} \nabla_i \bar v\tr \phi_v(x^0).
\end{align*}
The above follows in a similar manner as Proposition $10.2$ of~\citet{Bhatnagar13SR}. On similar lines, one can see that  
\begin{align*}
\E\left[ \dfrac{\Delta^{(i)}}{\beta}(\bar u^+ - \bar u)\tr \phi_u(x^0) \left. \right| \theta,\lambda\right] 
\stackrel{\beta \rightarrow 0}{\longrightarrow} \nabla_i \bar u\tr \phi_u(x^0).
\end{align*}
Thus,~\eqref{eq:actor-sf-update} can be seen to be a discretization of the ODE~\eqref{eq:theta-ode} and the rest of the analysis follows in a similar manner as in the SPSA proof.  
$\hfill{\blacksquare}$
\end{proof}

\subsubsection{Convergence of the Second-Order Algorithms: RS-SPSA-N and RS-SF-N}
\label{subsubsec:second-proofs}

Convergence analysis of the second-order algorithms involves the same steps as that of the first-order algorithms. In particular, the first step involving the TD-critic and the third step involving the analysis of $\lambda$-recursion follow along similar lines as earlier, whereas $\theta$-recursion analysis in the second step differs significantly. \\

\noindent
{\bf Step 2: (Analysis of $\theta$-recursion for RS-SPSA-N and RS-SF-N)} Since the policy parameter is updated in the descent direction with a Newton decrement, the limiting ODE of the $\theta$-recursion for the second order algorithms is given by
\begin{align}
\label{eq:theta-second-ode}
\dot{\theta}_t = \check{\Gamma}\left(\Upsilon\big(\nabla_{\theta}^2 L(\theta_t, \lambda)\big)^{-1} \nabla_\theta L(\theta_t, \lambda)\right),
\end{align}
where $\check{\Gamma}$ is as before (see \eqref{eq:Pi-bar-operator}). Let 
\begin{equation*}
\Z_\lambda = \left \{ \theta \in C: - \nabla_{\theta} L (\theta_t, \lambda)^T \Upsilon\big(\nabla^2_\theta L(\theta_t, \lambda)\big)^{-1} \nabla_\theta L(\theta_t, \lambda) = 0 \right \}.
\end{equation*} 
denote the set of asymptotically stable equilibrium points of the ODE~\eqref{eq:theta-second-ode} and $\Z_\lambda^\varepsilon$ its $\varepsilon$-neighborhood. Then, we have the following analogue of Theorem \ref{thm:spsa-theta-convergence} for the RS-SPSA-N and RS-SF-N algorithms:
\begin{theorem}
\label{thm:spsa-sf-n-theta-convergence}
Under (A1)-(A5), for any given Lagrange multiplier $\lambda$ and $\varepsilon > 0$, there exists $\beta_0 >0$ such that for all $\beta \in (0, \beta_0)$, $\theta_n \rightarrow \theta^* \in \Z^\varepsilon_{\lambda}$ almost surely.
\end{theorem}
\subsection*{\textbf{Proof of Theorem~\ref{thm:spsa-sf-n-theta-convergence} for RS-SPSA-N}}
Before we prove Theorem \ref{thm:spsa-sf-n-theta-convergence}, we establish that the Hessian estimate $H_n$ in \eqref{eq:hessian-update-spsa} converges almost surely to the true Hessian $\nabla^2_{\theta} L(\theta_n, \lambda)$ in the following lemma.
\begin{lemma}
\label{lemma:spsa-n}
With $\beta \rightarrow 0$, for all $i, j \in \{1, \ldots, \kappa_1 \}$, we have the following claims with probability one:
\begin{enumerate}[\bfseries(i)]
\item $\left \| \dfrac{L(\theta_n + \beta \Delta_n + \beta \widehat\Delta_n, \lambda) - L(\theta_n,\lambda)}{\beta^2 \Delta_n^{(i)} \widehat\Delta_n^{(j)}} - \nabla^2_{\theta_n^{(i, j)}} L(\theta_n, \lambda) \right \| \rightarrow 0,
$\\[1ex]
 \item $\left \| \dfrac{L(\theta_n + \beta \Delta_n + \beta \widehat\Delta_n, \lambda) - L(\theta_n,\lambda)}{\beta \widehat\Delta_n^{(i)}} - \nabla_{\theta_n^{(i)}} L(\theta_n, \lambda) \right \| \rightarrow 0,$\\[1ex]
\item $\left \| H^{(i, j)} - \nabla^2_{\theta_n^{(i, j)}} L(\theta_n, \lambda) \right \| \rightarrow 0,
$\\[1ex]
\item $\left \| M - \Upsilon(\nabla^2_{\theta_n} L(\theta_n, \lambda))^{-1} \right \| \rightarrow 0.
$
\end{enumerate}
\end{lemma}

\begin{proof}
The proofs of the above claims follow from Propositions 10.10, 10.11 and Lemmas 7.10 and 7.11 of \cite{Bhatnagar13SR}, respectively.  
$\hfill{\blacksquare}$
\end{proof}

\begin{proof}{\bf (Theorem~\ref{thm:spsa-sf-n-theta-convergence} for RS-SPSA-N)}
As in the case of the first order methods, due to timescale separation, we can treat $\lambda_n \equiv \lambda$, a constant and use the converged TD-parameters to arrive at the following equivalent update rules  for the Hessian recursion \eqref{eq:hessian-update-spsa} and $\theta$-recursion \eqref{eq:actor-spsa-n-update}:
\begin{align*}
H^{(i, j)}_{n+1}=&H^{(i, j)}_n + \zeta'_2(n)\bigg[\dfrac{\big(1 + \lambda_n (\bar v_n + \bar v^+_n)\tr \phi_v(x^0) \big)(\bar v_n-\bar v^+_n)\tr \phi_v(x^0)}{\beta^2 \Delta^{(i)}_n\widehat\Delta^{(j)}_n}+ \dfrac{\lambda (\bar u^+_n-\bar u_n)\tr \phi_u(x^0)}{\beta^2 \Delta^{(i)}_n\widehat\Delta^{(j)}_n} - H^{(i, j)}_n \bigg],\\
\theta_{n+1}^{(i)}=& \Gamma_i\bigg[\theta_n^{(i)} + \zeta_2(n)\sum\limits_{j = 1}^{\kappa_1} M^{(i, j)}_n\Big(\dfrac{\big(1+2\lambda \bar v_n\tr \phi_v(x^0)\big)(\bar v^+_n - \bar v_n)\tr \phi_v(x^0)}{\beta \Delta_n^{(j)}} -\dfrac{\lambda(\bar u^+_n - \bar u_n)\tr \phi_u(x^0)}{\beta \Delta_n^{(j)}}\Big)\bigg].
\end{align*}
In lieu of Lemma \ref{lemma:spsa-n}, the $\theta$-recursion above is equivalent to the following:

\begin{align}
 \theta^{(i)}_{n+1}  &=  \bar\Gamma_i \bigg( \theta^{(i)}_n + \zeta_2(n) \big(\nabla_{\theta}^2 L(\theta_n, \lambda)\big)^{-1} \nabla_\theta L(\theta_n, \lambda)\bigg).
\end{align}
The above can be seen as a discretization of the ODE \eqref{eq:theta-second-ode}, with $\Z_\lambda$ serving as its asymptotically stable attractor. The rest of the claim follows in a similar manner as Theorem \ref{thm:spsa-theta-convergence}.
$\hfill{\blacksquare}$
\end{proof}


\subsection*{\textbf{Proof of Theorem~\ref{thm:spsa-sf-n-theta-convergence} for RS-SF-N}} 

\begin{proof}
We first establish the following result for the gradient and Hessian estimators employed in RS-SF-N:

\begin{lemma}
\label{lemma:sf-n}
With $\beta \rightarrow 0$, we have the following claims with probability one:
\begin{enumerate}[\bfseries(i)]
 \item $\Bigg\| E \left[\frac{1}{\beta^2}
\bar{H}(\Delta_n) (L(\theta_n +\beta \Delta_n,\lambda) -
L(\theta_n,\lambda))\mid \theta_n,\lambda \right]
- \nabla^2_{\theta} L(\theta_n,\lambda)
 \Bigg\| \rightarrow 0.$\\[1ex]
\item $\| E\left[\dfrac{1}{\beta} \Delta_n (L(\theta_n+\beta\Delta_n,\lambda)
-L(\theta_n,\lambda))\mid \theta_n,\lambda \right] -
\nabla_{\theta} L(\theta_n,\lambda) \| \rightarrow 0.$
\end{enumerate}
\end{lemma}

\begin{proof}
The proofs of the above claims follow from Propositions 10.1 and 10.2 of \cite{Bhatnagar13SR}, respectively. 
$\hfill{\blacksquare}$
\end{proof}

\noindent 
The rest of the analysis is identical to that of RS-SPSA-N. 
$\hfill{\blacksquare}$
\end{proof}

\begin{remark}(\textbf{On Convergence Rate.})
In the above, we established asymptotic limits for all our algorithms using the ODE approach. To the best of our knowledge, there are no convergence rate results available for multi-timescale stochastic approximation schemes, and hence, for actor-critic algorithms. This is true even for the actor-critic algorithms that do not incorporate any risk criterion. 
In \cite{konda2004convergence}, the authors provide asymptotic convergence rate results for {\em linear} two-timescale recursions.  
It would be an interesting direction for future research to obtain concentration bounds for general (non-linear) two-timescale schemes.

While a rigorous analysis on convergence rate of our proposed schemes is difficult, one could make a few concessions and use the following argument to see that the SPSA-based algorithms converge quickly:
In order to analyse the rate of convergence of $\theta$-recursion, assume (for sufficiently large $n$) that the TD-critic has converged in the inner-loop. This is because, the trajectory lengths $m_n \rightarrow \infty$ as $n \rightarrow \infty$ and under appropriate step-size settings (or with iterate averaging) one can obtain convergence rate of the order $O\left(1/\sqrt{n}\right)$ on the root mean square error of TD (see \cite{korda2014td}). Now, if one holds $\lambda$ fixed, then invoking asymptotic normality results for SPSA (see Proposition 2 in  \cite{Spall92MS}) it can be shown that \\$n^{1/3}(\theta_n - \theta^{\lambda})$ is asymptotically normal, where $\theta^{\lambda}$ is a limit point in the set $\Z_\lambda$. Similar results also hold for second-order SPSA variants (cf. Theorem 3a in \cite{spall2000adaptive}). Both the aforementioned claims are proved using a well-known result on asymptotic normality of stochastic approximation schemes due to Fabian \cite{fabian1968asymptotic}. 

The second-order schemes such as RS-SPSA-N score over their first order counterpart RS-SPSA-G from a asymptotic normality results perspective. This is because obtaining the optimal convergence rate for RS-SPSA-G requires that the step-size $\zeta_2(n)$ is set to $\zeta_2(0)/n$ where $\zeta_2(0) > 1/\lambda_{\min}(\nabla^2_\theta L(\theta^{\lambda},\lambda))$, whereas there is no such constraint for the second-order algorithm RS-SPSA-N. Here $\lambda_{\min}(A)$ denotes the minimum eigenvalue of the matrix $A$. The reader is referred to \cite{dippon1997weighted} for a detailed discussion on convergence rate of (one timescale) SPSA-based schemes using asymptotic mean-square error.
\end{remark}

\begin{remark}(\textbf{Unstable Equilibria.})
The limit set $\Z_\lambda$ contains both stable and unstable equilibria and the $\theta$-recursion can possibly end up in a unstable equilibrium point.
One may avoid this situation by including additional noise in the randomized policy that drives the $\theta$-recursion. For instance, define a $\eta$-offset policy as  
$$\hat{\mu}(a \mid x) = \dfrac{\mu(a \mid x) + \eta}{\sum \limits_{a' \in \A(x)} \left (\mu(a' \mid x) + \eta \right )}.$$
The above policy can be used in place of the regular $\mu(\cdot\mid x)$, so that the algorithm is pulled away from an unstable equilibria.
 Providing theoretical guarantees for such a scheme is non-trivial and we have left it for future work.
\end{remark}


\section{Convergence Analysis of the Average Reward Risk-Sensitive Actor-Critic Algorithm}
\label{sec:average-analysis}

As in the discounted setting, we use the ODE approach~\citep{borkar2008stochastic} to analyze the convergence of our average reward risk-sensitive actor-critic algorithm. The proof involves three main steps:

\begin{enumerate}
\item The first step is the convergence of $\rho$, $\eta$, $V$, and $U$, for any fixed policy $\theta$ and Lagrange multiplier $\lambda$. This corresponds to a TD(0) (with extension to $\eta$ and $U$) proof. Using arguments similar to that in Step 2 of the proof of RS-SPSA-G, one can show that the $\theta$ and $\lambda$ recursions track $\dot \theta_t =0$ and $\dot \lambda_t=0$, when viewed from the TD critic timescale $\{\zeta_3(t)\}$. Thus, the policy $\theta$  and Lagrange multiplier $\lambda$ are assumed to be constant in the analysis of the critic recursion. 
\item The second step is to show the convergence of $\theta_n$ to an $\varepsilon$-neighborhood $\Z_\lambda^\varepsilon$ of the set of asymptotically stable equilibria $\Z_\lambda$ of ODE
\begin{equation}
\label{eq:average-theta-ode}
\dot{\theta}_t=\check\Gamma\big(\nabla L(\theta_t,\lambda)\big),
\end{equation}
where the projection operator $\check\Gamma$ ensures that the evolution of $\theta$ via the ODE~\eqref{eq:average-theta-ode} stays within the compact and convex set $\Theta\subset \R^{\kappa_1}$ and is defined in \eqref{eq:Pi-bar-operator}.
Again here it is assumed that $\lambda$ is fixed because $\theta$-recursion is on a faster time-scale than $\lambda$'s. 
\item The final step is the convergence of $\lambda$ and showing that the whole algorithm converges to a local saddle point of $L(\theta,\lambda)$. where the limit is shown to satisfy the variance constraint in \eqref{eq:average-risk-measure}. 
\end{enumerate}


\noindent 
{\bf Step~1: Critic's Convergence}

\begin{lemma}
\label{critic-convergence}
For any given policy $\mu$, $\{\widehat{\rho}_n\}$, $\{\widehat{\eta}_n\}$, $\{v_n\}$, and $\{u_n\}$, defined in Algorithm~\ref{algo:average-AC} and by the critic recursion~\eqref{eq:average-critic-update} converge to $\rho(\mu)$, $\eta(\mu)$, $v^\mu$, and $u^\mu$ almost surely, where $v^\mu$ and $u^\mu$ are the unique solutions to 
\begin{equation}
\label{eq:diff-TD}
\Phi_v\tr \boldsymbol{D}^\mu\Phi_vv^\mu=\Phi_v\tr\boldsymbol{D}^\mu T^\mu_v(\Phi_vv^\mu), \quad\quad\quad \Phi_u\tr \boldsymbol{D}^\mu\Phi_uu^\mu=\Phi_u\tr\boldsymbol{D}^\mu T^\mu_u(\Phi_uu^\mu),
\end{equation}
respectively. In~\eqref{eq:diff-TD}, $\boldsymbol{D}^\mu$ denotes the diagonal matrix with entries $d^\mu(x)$ for all $x\in\X$, and $T^\mu_v$ and $T^\mu_u$ are the Bellman operators for the differential value and square value functions of policy $\mu$, defined as 
\begin{equation}
\label{eq:diff-Bellman-operator}
T_v^\mu J = \boldsymbol{r}^\mu - \rho(\mu)\boldsymbol{e} + \boldsymbol{P}^\mu J, \quad\quad\quad\quad T_u^\mu J = \boldsymbol{R}^\mu\boldsymbol{r}^\mu - \eta(\mu)\boldsymbol{e} + \boldsymbol{P}^\mu J,
\end{equation}
where $\boldsymbol{r}^\mu$ and $\boldsymbol{P}^\mu$ are the reward vector and transition probability matrix of policy $\mu$, $\boldsymbol{R}^\mu=diag(\boldsymbol{r}^\mu)$, and $\boldsymbol{e}$ is a vector of size $n$ (the size of the state space $\X$) with elements all equal to one.
\end{lemma}
\begin{proof}
The proof follows in a similar manner as Lemma~5 in~\cite{bhatnagar2009natural}.  
$\hfill{\blacksquare}$
\end{proof}

\noindent
{\bf Step~2: Actor's Convergence}

Let $\Z_\lambda=\big\{\theta\in C:\check\Gamma\big(-\nabla L(\theta,\lambda)\big)=0\big\}$ denote the set of asymptotically stable equilibrium points of the ODE~\eqref{eq:average-theta-ode} and $\Z_\lambda^\varepsilon=\big\{\theta\in C:||\theta-\theta_0||<\varepsilon,\theta_0\in\Z_\lambda\big\}$ denote the set of points in the $\varepsilon$-neighborhood of $\Z_\lambda$.
The main result regarding the convergence of the policy parameter in \eqref{eq:average-actor-update-theta-L} is as follows:
\begin{theorem}
\label{actor-convergence}
Assume (A1)-(A4). Then, for a given $\varepsilon >0,\;\exists\beta>0$ such that if $\sup_{\theta} \|\B(\theta,\lambda)\|<\beta$, then $\theta_n$ governed by \eqref{eq:average-actor-update-theta-L} converges almost surely to $\Z^\varepsilon_\lambda$ as $n\rightarrow\infty$.
\end{theorem}

\begin{proof}
 Let $\F(n)=\sigma(\theta_m,m\leq n)$ denote a sequence of $\sigma$-fields. We have
\begin{align*}
\theta_{n+1} &= \Gamma\Big(\theta_n-\zeta_2(n)\big(-\delta_n\psi_n+\lambda(\epsilon_n\psi_n-2\widehat{\rho}_{n+1}\delta_n\psi_n)\big)\Big) \\
&= \Gamma\big(\theta_n+\zeta_2(n)(1+2\lambda\widehat{\rho}_{n+1})\delta_n\psi_n-\zeta_2(n)\lambda\epsilon_n\psi_n\big) \\
&= \Gamma\bigg(\theta_n-\zeta_2(n)\Big[1+2\lambda\Big(\big(\widehat{\rho}_{n+1}-\rho(\theta_n)\big)+\rho(\theta_n)\Big)\Big]\E\big[\delta^{\theta_n}\psi_n|\F(n)\big] \\
&\hspace{0.5in}-\zeta_2(n)\Big[1+2\lambda\Big(\big(\widehat{\rho}_{n+1}-\rho(\theta_n)\big)+\rho(\theta_n)\Big)\Big]\Big(\delta_n\psi_n-\E\big[\delta_n\psi_n|\F(n)\big]\Big) \\
&\hspace{0.5in}-\zeta_2(n)\Big[1+2\lambda\Big(\big(\widehat{\rho}_{n+1}-\rho(\theta_n)\big)+\rho(\theta_n)\Big)\Big]\E\big[(\delta_n-\delta^{\theta_n})\psi_n|\F(n)\big] \\ 
&\hspace{0.5in}+\zeta_2(n)\lambda\E\big[\epsilon^{\theta_n}\psi_n|\F(n)\big] + \zeta_2(n)\lambda\Big(\epsilon_n\psi_n-\E\big[\epsilon_n\psi_n|\F(n)\big]\Big) \\
&\hspace{0.5in}+ \zeta_2(n)\lambda\E\big[(\epsilon_n-\epsilon^{\theta_n})\psi_n|\F(n)\big] \bigg). 
\end{align*}
By setting $\xi_n=\widehat{\rho}_{n+1}-\rho(\theta_n)$, we may write the above equation as
\begin{align}
\label{eq:theta-average-unrolled}
\theta_{n+1} &= \Gamma\bigg(\theta_n-\zeta_2(n)\big[1+2\lambda\big(\xi_n+\rho(\theta_n)\big)\big]\E\big[\delta^{\theta_n}\psi_n|\F(n)\big] \\
&\hspace{0.5in}-\zeta_2(n)\big[1+2\lambda\big(\xi_n+\rho(\theta_n)\big)\big]\underbrace{\Big(\delta_n\psi_n-\E\big[\delta_n\psi_n|\F(n)\big]\Big)}_{*} \nonumber\\
&\hspace{0.5in}-\zeta_2(n)\big[1+2\lambda\big(\xi_n+\rho(\theta_n)\big)\big]\underbrace{\E\big[(\delta_n-\delta^{\theta_n})\psi_n|\F(n)\big]}_{+} \nonumber\\ 
&\hspace{0.5in}+\zeta_2(n)\lambda\E\big[\epsilon^{\theta_n}\psi_n|\F(n)\big] + \zeta_2(n)\lambda\underbrace{\Big(\epsilon_n\psi_n-\E\big[\epsilon_n\psi_n|\F(n)\big]\Big)}_{*} \\
&\hspace{0.5in}+ \zeta_2(n)\lambda\underbrace{\E\big[(\epsilon_n-\epsilon^{\theta_n})\psi_n|\F(n)\big]}_{+} \bigg). \nonumber 
\end{align}
Since Algorithm~\ref{algo:average-AC} uses an unbiased estimator for $\rho$, we have $\widehat{\rho}_{n+1}\rightarrow\rho(\theta_n)$, and thus, $\xi_n\rightarrow 0$. The terms $(+)$ asymptotically vanish in lieu of Lemma~\ref{critic-convergence} (Critic convergence). Finally the terms $(*)$ can be seen to vanish using standard martingale arguments (cf.~Theorem~2 in~\cite{bhatnagar2009natural}). Thus,~\eqref{eq:theta-average-unrolled} can be seen to be equivalent in an asymptotic sense to 
\begin{equation}
\label{eq:temp1}
\theta_{n+1} = \Gamma\Big(\theta_n-\zeta_2(n)\big[1+2\lambda\rho(\theta_n)\big]\E\big[\delta^{\theta_n}\psi_n|\F(n)\big]+\zeta_2(n)\lambda\E\big[\epsilon^{\theta_n}\psi_n|\F(n)\big]\Big). 
\end{equation}
From the foregoing,~it can be seen that the actor recursion in \eqref{eq:average-actor-update-theta-L} asymptotically tracks the stable fixed points of the ODE
\begin{equation}
\label{eq:theta-average-ode1}
\dot \theta_{t} = \check\Gamma\Big( \nabla L(\theta_t,\lambda) + \B(\theta_t,\lambda)\Big). 
\end{equation}

Note that the bias of Algorithm~\ref{algo:average-AC} in estimating $\nabla L(\theta,\lambda)$ is (see Lemma~\ref{bias-average})
\begin{align*}
\B(\theta,\lambda)=&\sum_xd^\theta(x)\Big\{-\big(1+2\lambda\rho(\theta)\big)\big[\nabla\bar{V}^\theta(x)-\nabla v^{\theta\top}\phi_v(x)\big] + \lambda\big[\nabla\bar{U}^\theta(x) - \nabla u^{\theta\top}\phi_u(x)\big]\Big\}.
\end{align*}

So, if the bias $\sup_{\theta} \|\B(\theta,\lambda)\| \rightarrow 0$, the trajectories \eqref{eq:theta-average-ode1} converge to those of \eqref{eq:theta-ode} uniformly on compacts for the same initial condition and the claim follows. 
$\hfill{\blacksquare}$
\end{proof}


\noindent
{\bf Step~3: $\lambda$ Convergence and Overall Convergence of the Algorithm} \\

\noindent
As in the discounted setting, we first show that the $\lambda$-recursion converges and then prove convergence to a local saddle point of $L(\theta,\lambda)$. 
Consider the ODE
\begin{align}
\label{eq:lambda-ode-average}
\dot \lambda_t = \check\Gamma_\lambda\big(\Lambda(\theta^{\lambda_t}) - \alpha\big),
\end{align}
where $\check \Gamma_\lambda$ is a projection operator that forces the evolution of $\lambda$ via \eqref{eq:lambda-ode} is within $[0,\lambda_{\max}]$ and is defined in \eqref{eq:Pi-bar-operator-lambda}.


\begin{theorem}
\label{theorem:average-lambda}
$\lambda_n \rightarrow \F$ almost surely as $t \rightarrow \infty$, where $\F \stackrel{\triangle}{=}\big\{\lambda\mid \lambda \in [0,\lambda_{\max}], \check\Gamma_\lambda\big(\Lambda(\theta^\lambda) - \alpha\big)=0,\;\theta^\lambda \in \Z_\lambda\big\}$.
\end{theorem}

\begin{proof}
The proof follows in a similar manner as that of Theorem~3 in~\cite{bhatnagar2012online}.
$\hfill{\blacksquare}$
\end{proof} 

As in the discounted setting, the following proposition claims that the limit $\theta^{\lambda^*}$ corresponding to $\lambda^*$ satisfies the variance constraint in \eqref{eq:average-risk-measure}, i.e.,
\begin{proposition}
\label{prop:feasible-average}
 For any $\lambda^*$ in $\hat\F \stackrel{\triangle}{=}\big\{\lambda\mid \lambda \in [0,\lambda_{\max}),\;\check\Gamma_\lambda\big[ \Lambda^{\theta^\lambda}(x^0)-\alpha\big]=0,\;\theta^\lambda \in \Z_\lambda \big\}$, the corresponding limiting point $\theta^{\lambda^*}$ satisfies the variance constraint $\Lambda^{\theta^{\lambda^*}}(x^0) \le \alpha$.
\end{proposition}

Using arguments similar to that used to prove convergence of RS-SPSA-G, it can be shown that that the ODE \eqref{eq:lambda-ode-average} is equivalent to $\dot \lambda_t = \check\Gamma_\lambda\big[\nabla_\lambda  L(\theta^{\lambda_t},\lambda_t)\big]$ and thus, the actor parameters $(\theta_n,\lambda_n)$ updated according to \eqref{eq:average-actor-update-theta-L} converge to a (local) saddle point $(\theta^{\lambda^*},\lambda^*)$ of $L(\theta,\lambda)$. Morever, the limiting point $\theta^{\lambda^*}$ satisfies the variance constraint in \eqref{eq:average-risk-measure}.




\section{Experimental Results}
\label{sec:simulation}

We evaluate our algorithms in the context of a traffic signal control application. 
 The objective in our formulation is to minimize the total number of vehicles in the system, which indirectly minimizes the delay experienced by the system. The motivation behind using a risk-sensitive control strategy is to reduce the variations in the delay experienced by road users. 


\subsection{Implementation}

\tikzset{roads/.style={line width=0.2cm}}

\begin{figure}
\centering
        \includegraphics[width=3in,height=2in]{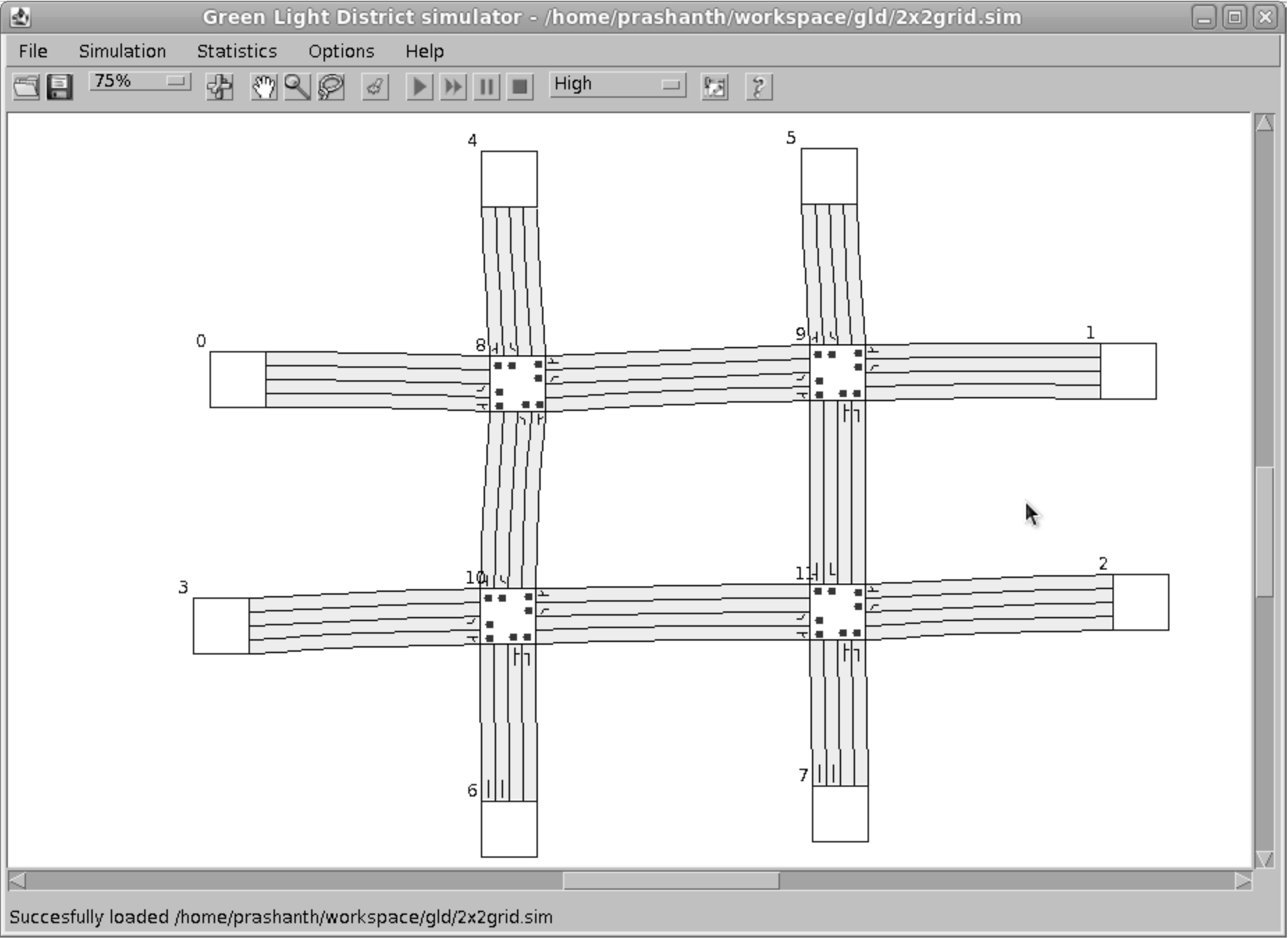}
  \caption{The 2x2-grid network used in our traffic signal control experiments.}
\label{fig:roadnets}
  \end{figure}

We consider both infinite horizon discounted and average settings for the traffic signal control MDP, formulated as in~\cite{prashanth2011reinforcement}. We briefly recall their formulation here: The state at each time $t$, $x_n$, is the vector of queue lengths and elapsed times and is given by $x_n = (q_1(n), \ldots, q_N(n), t_1(n), \ldots, t_N(n))$, where $N$ is the number of signalled lanes in the road network considered. Here $q_i$ and $t_i$ denote the queue length and elapsed time since the signal turned to red on lane $i$. The actions $a_n$ belong to the set of feasible sign configurations. The single-stage cost function $h(x_n)$ is defined as follows:

\begin{align}
\label{eq:cost-traffic}
h(x_n)  =&  r_1 * \big[ \sum_{i \in I_p} r_2 * q_i(n) + \sum_{i \notin I_p} s_2 * q_i(n)\big] + s_1 * \big[\sum_{i \in I_p} r_2 * t_i(n) + \sum_{i \notin I_p} s_2 * t_i(n) \big],
\end{align}
where $r_i,s_i \ge 0$ such that $r_i + s_i =1$ for $i=1,2$ and $r_2 > s_2$. The set $I_p$ is the set of prioritized lanes in the  road network considered. While the weights $r_1, s_1$ are used to differentiate between the queue length and elapsed time factors, the weights $r_2,s_2$ help in prioritization of traffic. 

Given the above traffic control setting, we aim to minimize both the long run discounted and average sum of the cost function $h(x_n)$. We implement the following algorithms using the Green Light District (GLD) simulator~\cite{GLDSim}\footnote{We would like to point out that the experimental setting involves 'costs' and not 'rewards' and the algorithms implemented should be understood as optimizing a negative reward.
}: 

\paragraph{\bf Discounted Setting}
\begin{enumerate}
\item {\bf\em SPSA-G}: This is a first-order risk-neutral algorithm with SPSA-based gradient estimates that updates the parameter $\theta$ as follows:
\begin{align*}
\theta_{n+1}^{(i)} &= \Gamma_i\left(\theta_n^{(i)} + \frac{\zeta_2(n)}{\beta\Delta_n^{(i)}}(v^+_n - v_n)\tr \phi_v(x^0)\right),
\end{align*} 
where the critic parameters $v_n, v^+_n$ are updated according to~\eqref{eq:critic-discounted}. Note that this is a two-timescale algorithm with a TD critic on the faster timescale and the actor on the slower timescale. Unlike RS-SPSA-G, this algorithm, being risk-neutral, does not involve the Lagrange multiplier recursion.

\item {\bf\em SF-G}: This is a first-order risk-neutral algorithm that is similar to SPSA-G, except that the gradient estimation scheme used here is based on the smoothed functional (SF) technique. The update of the policy parameter in this algorithm is given by
\begin{align*}
\theta_{n+1}^{(i)} &= \Gamma_i\left(\theta_n^{(i)} + \zeta_2(n)\Big(\frac{\Delta_n^{(i)}}{\beta}(v^+_n - v_n)\tr \phi_v(x^0)\Big)\right).
\end{align*}
\item {\bf\em SPSA-N}: This is a risk-neutral algorithm and is the second-order counterpart of SPSA-G. The Hessian update in this algorithm is as follows: For $i,j=1,\ldots, \kappa_1$, $i< j$, the update is
\begin{align}
\label{eq:hessian-update-no-risk-spsa-n}
H^{(i, j)}_{n+1}= H^{(i, j)}_n + \zeta'_2(n)\bigg[&\dfrac{(v_n-v^+_n)\tr \phi_v(x^0)}{\beta^2 \Delta^{(i)}_n\widehat\Delta^{(j)}_n}  - H^{(i, j)}_n \bigg],
\end{align}
and for $i > j$, we set $H^{(i, j)}_{n+1} = H^{(j, i)}_{n+1}$. As in RS-SPSA-N, let $M_n \stackrel{\triangle}{=} H_n^{-1}$, where $H_n = \Upsilon\big([H^{(i,j)}_n]_{i,j = 1}^{|\kappa_1|}\big)$. The actor updates the parameter $\theta$ as follows:
\begin{align}
\label{eq:actor-no-risk-spsa-n-update}
\theta_{n+1}^{(i)}= \Gamma_i\bigg[\theta_n^{(i)} + \zeta_2(n)\sum\limits_{j = 1}^{\kappa_1} M^{(i, j)}_n\Big(&\dfrac{(v^+_n - v_n)\tr \phi_v(x^0)}{\beta \Delta_n^{(j)}} \Big)\bigg].
\end{align}
The rest of the symbols, including the critic parameters, are as in RS-SPSA-N.

\item {\bf\em SF-N}: This is a risk-neutral algorithm and is the second-order counterpart of SF-G. It updates the Hessian and the actor as follows: For $i,j,k=1,\ldots, \kappa_1$, $j< k$, the Hessian update is 
\begin{align*}
\text{\bf Hessian:} \quad\quad H^{(i, i)}_{n + 1} &= H^{(i, i)}_n + \zeta'_2(n)\bigg[\dfrac{\big(\Delta^{(i)^2}_n-1\big)}{\beta^2}(v_n-v^+_n)\tr \phi_v(x^0)  - H^{(i, i)}_n \bigg],\\
H^{(j, k)}_{n + 1} &= H^{(j, k)}_n + \zeta'_2(n)\bigg[\dfrac{\Delta^{(j)}_n\Delta^{(k)}_n}{\beta^2}(v_n-v^+_n)\tr \phi_v(x^0)  - H^{(j, k)}_n \bigg],
\end{align*}
and for $j > k$, we set $H^{(j, k)}_{n+1} = H^{(k, j)}_{n+1}$. As before, let $M_n \stackrel{\triangle}{=} H_n^{-1}$, with $H_n$ formed as in SPSA-N. Then, the actor update for the parameter $\theta$ is as follows:
\begin{equation*}
\text{\bf Actor:} \quad\quad \theta_{n+1}^{(i)}= \Gamma_i\bigg[\theta_n^{(i)} + \zeta_2(n)\sum\limits_{j = 1}^{\kappa_1} M^{(i, j)}_n\frac{ \Delta_n^{(j)}}{\beta}(v^+_n - v_n)\tr \phi_v(x^0) \bigg].
\end{equation*}
The rest of the symbols, including the critic parameters, are as in RS-SPSA-N.

\item {\bf\em RS-SPSA-G}: This is the first-order risk-sensitive actor-critic algorithm that attempts to solve~\eqref{eq:average-risk-measure} and updates according to~\eqref{eq:actor-spsa-update}. 
\item {\bf\em RS-SF-G}: This is a first-order algorithm and the risk-sensitive variant of SF-G that updates the actor according to~\eqref{eq:actor-sf-update}. 
\item {\bf\em RS-SPSA-N}: This is a second-order risk-sensitive algorithm that estimates gradient and Hessian using SPSA and updates them according to~\eqref{eq:actor-spsa-n-update}. 
\item {\bf\em RS-SF-N}: This second-order risk-sensitive algorithm is the SF counterpart of RS-SPSA-N, and updates according to~\eqref{eq:actor-sf-n-update}. 
\end{enumerate}

\paragraph{\bf Average Setting}
\begin{enumerate}
\item {\bf\em AC}: This is an actor-critic algorithm that minimizes the long-run average sum of the single-stage cost function $h(x_n)$, without considering any risk criteria. This is similar to Algorithm~1 in~\citet{bhatnagar2009natural}.
\item {\bf\em RS-AC}: This is the risk-sensitive actor-critic algorithm that attempts to solve~\eqref{eq:average-risk-measure} and is described in Section~\ref{sec:average-alg}. 
\end{enumerate}

The underlying policy that guides the selection of the sign configuration in each of the algorithms above is a parameterized Boltzmann family and has the form
\begin{equation}
\mu_{\theta}(x,a) = \frac{e^{\theta^{\top} \phi_{x,a}}}{\sum_{a' \in {\A(x)}} e^{\theta^{\top} \phi_{x,a'}}},
\hspace{6pt} \forall x \in \X,\;\forall a \in \A.
\label{eq:pi-boltzmann}
\end{equation}
All our algorithms incorporate function approximation owing to the curse of dimensionality associated with larger road networks. For instance, assuming only $20$ vehicles per lane of a 2x2-grid network, the cardinality of the state space is approximately of the order $10^{32}$ and the situation is aggravated as the size of the road network increases. The choice of features used in each of our algorithms is as described in Section V-B of~\cite{prashanth2012threshold}. 

The experiments for each algorithm comprised of the following two phases:
\begin{description}
\item[{\bf Policy Search Phase:}] Here each iteration involved the simulation run with the nominal policy parameter $\theta$ as well as the perturbed policy parameter $\theta^+$ (algorithm-specific). We run each algorithm for $500$ iterations, where the run length for a particular policy parameter is $150$ steps.    
\item[{\bf Policy Test Phase:}] After the completion of the policy search phase, we freeze the policy parameter and run $50$ independent simulations with this (converged) choice of the parameter. The results presented subsequently are averages  over these $50$ runs.
\end{description}
  
Figure~\ref{fig:roadnets} shows a snapshot of the road network used for conducting the experiments from GLD simulator. Traffic is added to the network at each time step from the edge nodes. The spawn frequencies specify the rate at which traffic is generated at each edge node and follow a Poisson distribution. The spawn frequencies are set such that the proportion of the number of vehicles on the main roads (the horizontal ones in Fig. \ref{fig:roadnets}) to those on the side roads is in the ratio of $100:5$. This setting is close to what is observed in practice and has also been used for instance in~\cite{prashanth2011reinforcement,prashanth2012threshold}. In all our experiments, we set the weights in the single stage cost function~\eqref{eq:cost-traffic} as follows: $r_1 = r_2 = 0.5$ and $r_2=0.6, s_2=0.4$. For the SPSA and SF-based algorithms in the discounted setting, we set the parameter $\delta = 0.2$ and the discount factor $\gamma=0.9$. The parameter $\alpha$ in the formulations~\eqref{eq:average-risk-measure} and~\eqref{eq:discounted-risk-measure} was set to $20$. The step-size sequences are chosen as follows:
\begin{align}
\zeta_1(n)= \frac{1}{n}, \quad \zeta_2(n)= \frac{1}{n^{0.75}}, \quad \zeta'_2(n)= \frac{1}{n^{0.7}}, \quad \zeta_3(n)= \frac{1}{n^{0.66}}, \quad\quad n \ge 1.
\end{align}
Further, the constant $k$ related to $\zeta_4(n)$ in the risk-sensitive average reward algorithm is set to $1$. It is easy to see that the choice of step-sizes above satisfies (A4).
The projection operator $\Gamma_i$ was set to project the iterate $\theta^{(i)}$ onto the set $[0,10]$, for all $i=1,\ldots,\kappa_1$, while the projection operator for the Lagrange multiplier used the set $[0,1000]$. All the experiments were performed on a 2.53GHz Intel quad core machine with 3.8GB RAM.


\subsection{Results}
  
\pgfmathdeclarefunction{gauss}{2}{%
\pgfmathparse{1/(#2*sqrt(2*pi))*exp(-((x-#1)^2)/(2*#2^2))}%
}

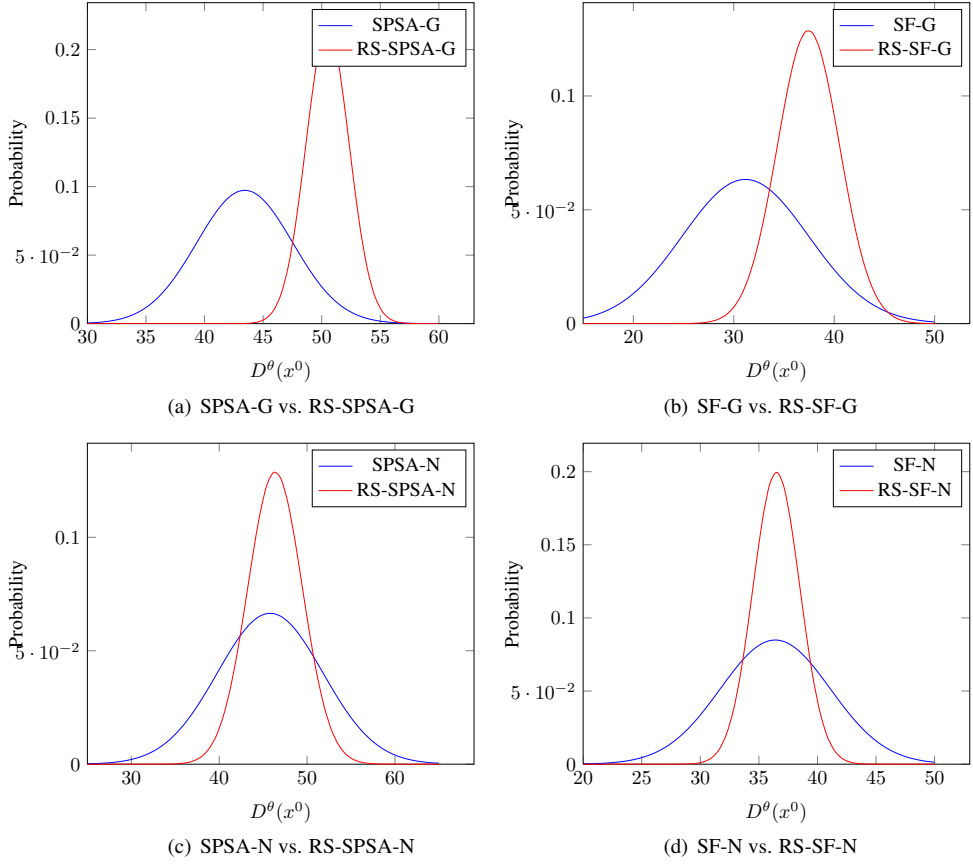
\begin{figure}
\centering
\begin{tabular}{cc}

\subfigure[SPSA-G vs. RS-SPSA-G]
{
      \label{fig:discounted_perf_spsa}
\hspace{-4em}\tabl{c}{\scalebox{0.75}{\begin{tikzpicture}
\begin{axis}[every axis plot post/.append style={
  mark=none,domain=30:60,samples=50,smooth}, 
  enlargelimits=upper, 
xlabel={$D^\theta(x^0)$},ylabel={Probability}
  ] 
  \addplot {gauss(43.43,4.1)};
  \addlegendentry{SPSA-G}
  \addplot {gauss(50.5,1.85)};
  \addlegendentry{RS-SPSA-G}
\end{axis}
\end{tikzpicture}}}
}
&
\subfigure[SF-G vs. RS-SF-G]
{
      \label{fig:discounted_perf_sf}
\hspace{-2em}\tabl{c}{\scalebox{0.75}{\begin{tikzpicture}
\begin{axis}[every axis plot post/.append style={
  mark=none,domain=15:50,samples=50,smooth}, 
  enlargelimits=upper, 
xlabel={$D^\theta(x^0)$},ylabel={Probability}
  ] 
  \addplot {gauss(31.16,6.3)};
  \addlegendentry{SF-G}
  \addplot {gauss(37.44,3.1)};
  \addlegendentry{RS-SF-G}
\end{axis}
\end{tikzpicture}}}
}
\\
\subfigure[SPSA-N vs. RS-SPSA-N]
{
\hspace{-4em}\tabl{c}{\scalebox{0.75}{\begin{tikzpicture}
\begin{axis}[every axis plot post/.append style={
  mark=none,domain=25:65,samples=50,smooth}, 
  enlargelimits=upper, 
xlabel={$D^\theta(x^0)$},ylabel={Probability}
  ] 
  \addplot {gauss(45.8,6)};
  \addlegendentry{SPSA-N}
  \addplot {gauss(46.36,3.1)};
  \addlegendentry{RS-SPSA-N}
\end{axis}
\end{tikzpicture}}}
}
&
\subfigure[SF-N vs. RS-SF-N]
{
\hspace{-2em}\tabl{c}{\scalebox{0.75}{\begin{tikzpicture}
\begin{axis}[every axis plot post/.append style={
  mark=none,domain=20:50,samples=50,smooth}, 
  enlargelimits=upper, 
xlabel={$D^\theta(x^0)$},ylabel={Probability}
  ] 
  \addplot {gauss(36.4,4.7)};
  \addlegendentry{SF-N}
  \addplot {gauss(36.5,2)};
  \addlegendentry{RS-SF-N}
\end{axis}
\end{tikzpicture}}}
}
\end{tabular}
\caption{Performance comparison in the discounted setting  using the distribution of $D^\theta(x^0)$.} 
\label{fig:discounted-normal}
\end{figure}

 \begin{figure}
    \centering
\begin{tabular}{cc}
    \subfigure[SPSA-G vs. RS-SPSA-G]{
      \label{fig:discounted_atwt_spsa}
\hspace{-4em}\tabl{c}{\scalebox{0.75}{\begin{tikzpicture}
\begin{axis}[xlabel={time},ylabel={TAR},smooth,legend pos=south east]
\addplot table[x index=0,y index=1,col sep=space,each nth point={100}] {results/tar/averaged_results_SPSA_TAR.dat};
\addlegendentry{SPSA-G}
\addplot table[x index=0,y index=1,col sep=space,each nth point={100}] {results/tar/averaged_results_RS_SPSA_TAR.dat};
\addlegendentry{RS-SPSA-G}
\end{axis}
\end{tikzpicture}}\\[1ex]}
}
&
    \subfigure[SF-G vs. RS-SF-G]{
      \label{fig:discounted_atwt_sf}
\hspace{-2em}\tabl{c}{\scalebox{0.75}{\begin{tikzpicture}
\begin{axis}[xlabel={time},ylabel={TAR},smooth,legend pos=south east]
\addplot table[x index=0,y index=1,col sep=space,each nth point={100}] {results/tar/averaged_results_SF_TAR.dat};
\addlegendentry{SF-G}
\addplot table[x index=0,y index=1,col sep=space,each nth point={100}] {results/tar/averaged_results_RS_SF_TAR.dat};
\addlegendentry{RS-SF-G}
\end{axis}
\end{tikzpicture}}\\[1ex]}
}
\\
    \subfigure[SPSA-N vs. RS-SPSA-N]{
\hspace{-4em}\tabl{c}{\scalebox{0.75}{\begin{tikzpicture}
\begin{axis}[xlabel={time},ylabel={TAR},smooth,legend pos=south east]
\addplot table[x index=0,y index=1,col sep=space,each nth point={100}] {results/tar/averaged_results_SPSA_N_TAR.dat};
\addlegendentry{SPSA-N}
\addplot table[x index=0,y index=1,col sep=space,each nth point={100}] {results/tar/averaged_results_RS_SPSA_N_TAR.dat};
\addlegendentry{RS-SPSA-N}
\end{axis}
\end{tikzpicture}}\\[1ex]}
}
&
    \subfigure[SF-N vs. RS-SF-N]{
\hspace{-2em}\tabl{c}{\scalebox{0.75}{\begin{tikzpicture}
\begin{axis}[xlabel={time},ylabel={TAR},smooth,legend pos=south east]
\addplot table[x index=0,y index=1,col sep=space,each nth point={100}] {results/tar/averaged_results_SF_N_TAR.dat};
\addlegendentry{SF-N}
\addplot table[x index=0,y index=1,col sep=space,each nth point={100}] {results/tar/averaged_results_RS_SF_N_TAR.dat};
\addlegendentry{RS-SF-N}
\end{axis}
\end{tikzpicture}}\\[1ex]}
}
\end{tabular}
\caption{Performance comparison of the algorithms in the discounted setting using the total arrived road users (TAR).}
\label{fig:tar}
\end{figure}
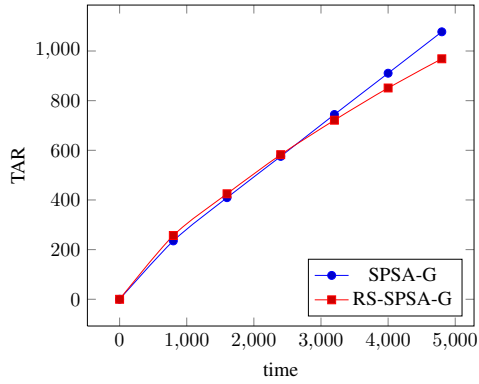
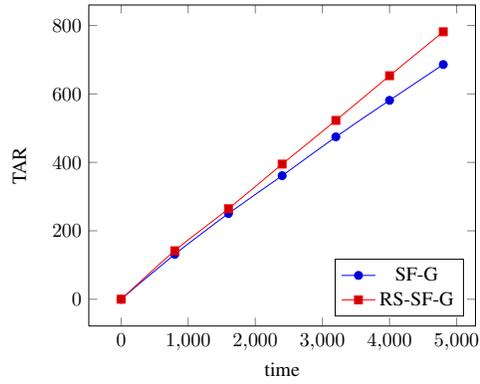
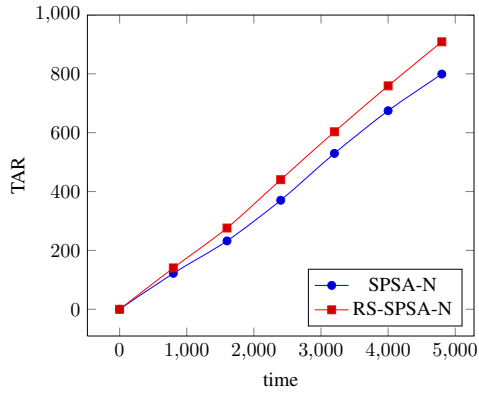
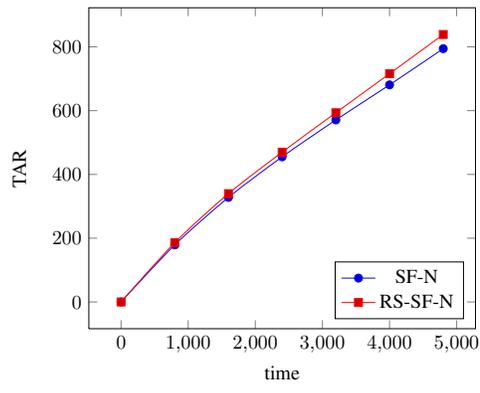

 \begin{figure}
    \centering
\tabl{c}{\scalebox{0.75}{\begin{tikzpicture}
\begin{axis}[xlabel={time},ylabel={AJWT},smooth,legend pos=south east]
\addplot table[x index=0,y index=1,col sep=space,each nth point={100}] {results/delay/averaged_results_SF.dat};
\addlegendentry{SF-G}
\addplot table[x index=0,y index=1,col sep=space,each nth point={100}] {results/delay/averaged_results_RS_SF.dat};
\addlegendentry{RS-SF-G}
\end{axis}
\end{tikzpicture}}\\[1ex]}
\caption{Performance comparison of the first-order SF-based algorithms, SF-G and RS-SF-G, using the average junction waiting time (AJWT).}
\label{fig:ajwt}
\end{figure}
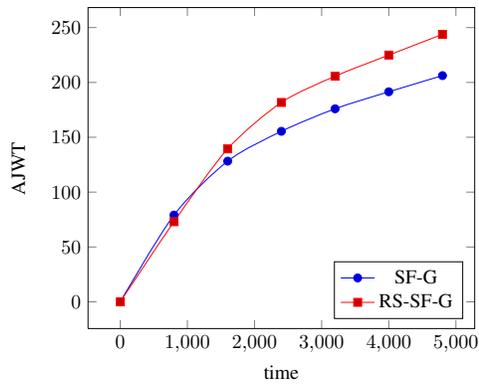

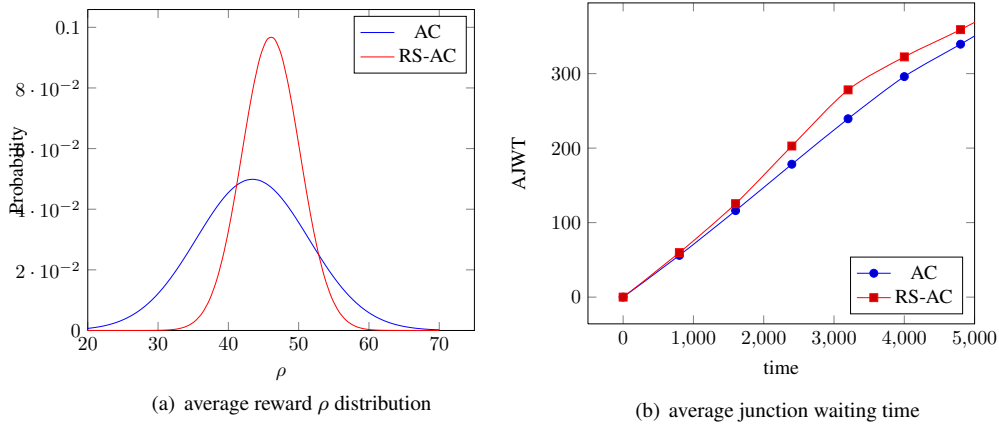
\begin{figure}
\centering
\begin{tabular}{cc}
\subfigure[average reward $\rho$ distribution]
{
  \label{fig:average_rho}
\hspace{-4em}\tabl{c}{\scalebox{0.75}{\begin{tikzpicture}
\begin{axis}[every axis plot post/.append style={
  mark=none,domain=20:70,samples=50,smooth}, 
  enlargelimits=upper, 
xlabel={$\rho$},ylabel={Probability}
  ] 
  \addplot {gauss(43.44,8)};
  \addlegendentry{AC}
  \addplot {gauss(46.07,4.12)};
  \addlegendentry{RS-AC}
\end{axis}
\end{tikzpicture}}}
}
&
\subfigure[average junction waiting time]
{
 \label{fig:atwt_perf_ac}
\hspace{-2em} \tabl{c}{\scalebox{0.75}{\begin{tikzpicture}
\begin{axis}[xlabel={time},ylabel={AJWT},smooth,legend pos=south east,xmax=5000]
\addplot table[x index=0,y index=1,col sep=space,each nth point={50}] {results/delay/averaged_results_AC_ATWT.dat};
\addlegendentry{AC}
\addplot table[x index=0,y index=1,col sep=space,each nth point={50}] {results/delay/averaged_results_RS_AC_ATWT.dat};
\addlegendentry{RS-AC}
\end{axis}
\end{tikzpicture}}\\[1ex]}
}
\end{tabular}
\caption{Performance comparison of the risk-neutral (AC) and risk-sensitive (RS-AC) average reward actor-critic algorithms using two different metrics.}
  \label{fig:average_perf}
\end{figure}

 \begin{figure}
    \centering
\begin{tabular}{cc}
    \subfigure[RS-SPSA-G]{
      \label{fig:conv-g}
\hspace{-4em}\tabl{c}{\scalebox{0.75}{\begin{tikzpicture}
\begin{axis}[xlabel={time},ylabel={$\theta_n^{(i)}$},smooth,legend pos=north east,line width=1pt]
\addplot[no markers,color=red] table[x index=0,y index=2,col sep=space] {results/RS_SPSA_theta_1_2x2grid_unif_alpha20.0_delta0.2.txt};
\addlegendentry{$\theta_n^{(2)}$}
\addplot[no markers,color=blue] table[x index=0,y index=7,col sep=space] {results/RS_SPSA_theta_1_2x2grid_unif_alpha20.0_delta0.2.txt};
\addlegendentry{$\theta_n^{(7)}$}
\end{axis}
\end{tikzpicture}}\\[1ex]}
}
&
    \subfigure[RS-SPSA-N]{
      \label{fig:conv-n}
\hspace{-2em}\tabl{c}{\scalebox{0.75}{\begin{tikzpicture}
\begin{axis}[xlabel={time},ylabel={$\theta_n^{(i)}$},smooth,legend pos=north east,smooth,line width=1pt]
\addplot[no markers,color=blue] table[x index=0,y index=5,col sep=tab] {results/RS_SPSA_N_theta_1_2x2grid_unif_alpha20.0_delta0.2.txt};
\addlegendentry{$\theta_n^{(2)}$}
\addplot[no markers,color=red] table[x index=0,y index=2,col sep=tab] {results/RS_SPSA_N_theta_1_2x2grid_unif_alpha20.0_delta0.2.txt};
\addlegendentry{$\theta_n^{(5)}$}
\end{axis}
\end{tikzpicture}}\\[1ex]}
}
\end{tabular}
\caption{Convergence of SPSA based algorithms in the discounted setting -- illustration using two (arbitrarily chosen) coordinates of the parameter $\theta$.}
\label{fig:conv}
\end{figure}
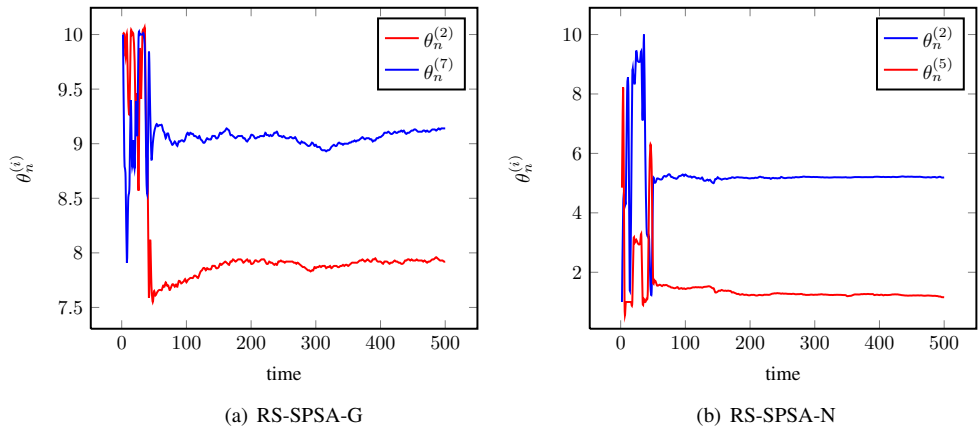

Figure~\ref{fig:discounted-normal} shows the distribution of the discounted cumulative reward $D^\theta(x^0)$ for the algorithms in the discounted setting. Figure~\ref{fig:tar} shows the total arrived road users (TAR) obtained for all the algorithms in the discounted setting, whereas Figure \ref{fig:ajwt} presents the average junction waiting time (AJWT) for the first-order SF-based algorithm RS-SF-G.\footnote{The AJWT performance of the other algorithms in the discounted setting is similar and the corresponding plots are omitted here.} TAR is a throughput metric that measures the number of road users who have reached their destination, whereas AJWT is a delay metric that quantifies the average delay experienced by the road users. 

The performance of the algorithms in the average setting is presented in Figure~\ref{fig:average_perf}. In particular, Figure \ref{fig:average_rho} shows the distribution of the average reward $\rho$, while Figure \ref{fig:atwt_perf_ac} presents the average junction waiting time (AJWT) for the average cost algorithms.

From Figures \ref{fig:discounted-normal} and \ref{fig:average_rho}, we notice that the risk-sensitive algorithms proposed in this paper result in a long-term (discounted or average) cost that is higher than their risk-neutral variants. However, from the empirical variance of the cost (both discounted as well as average) perspective, the risk-sensitive algorithms outperform their risk-neutral variants. Amongst our algorithms in the discounted setting, we observe that the second-order schemes (RS-SPSA-N and RS-SF-N) exhibit better results, though they involve an additional computational cost of inverting the Hessian at each time step. Further, from a traffic signal control application standpoint, we notice from the throughput (TAR) and delay (AJWT) plots (see Figures~\ref{fig:tar},~\ref{fig:ajwt} and~\ref{fig:atwt_perf_ac}), that the performance of the risk-sensitive algorithm variants is close to that of the corresponding risk-neutral algorithms in both the considered settings.

We observe that the policy parameter $\theta$ converges for the SPSA based algorithms in the discounted setting. This is illustrated in Figures~\ref{fig:conv-g} and~\ref{fig:conv-n}. Note that we established theoretical convergence of our algorithms earlier (see Sections~\ref{sec:SPSA-SF-proofs} and~\ref{sec:average-analysis}) and these plots confirm the same. Further, these plots also show that the transient period, i.e.,~the initial phase when $\theta$ has not converged, is short. Similar observations hold for the other algorithms as well. The results of this section indicate the rapid empirical convergence of our proposed algorithms. This observation coupled with the fact that they guarantee low variance of return, make them attractive for implementation in risk-constrained systems. 


\section{Conclusions and Future Work}
\label{sec:conclusions}

We proposed novel actor-critic algorithms for control in risk-sensitive discounted and average reward MDPs. All our algorithms involve a TD critic on the fast timescale, a policy gradient (actor) on the intermediate timescale, and a dual ascent for Lagrange multipliers on the slowest timescale. In the discounted setting, we pointed out the difficulty in estimating the gradient of the variance of the return and incorporated simultaneous perturbation based SPSA and SF approaches for gradient estimation in our algorithms. The average setting, on the other hand, allowed for an actor to employ compatible features to estimate the gradient of the variance. We provided proofs of convergence to locally (risk-sensitive) optimal policies for all the proposed algorithms. Further, using a traffic signal control application, we observed that our algorithms resulted in lower variance empirically as compared to their risk-neutral counterparts. 

As future work, it would be interesting to develop a risk-sensitive algorithm that uses a single trajectory in the discounted setting. Further, it would also be interesting to consider conditional value at risk (CVaR) as a measure of risk and develop a control algorithm that optimizes the return of a MDP with bounds on CVaR. The resulting algorithm could be applied for portfolio optimization in a financial application. An orthogonal direction of future research is to obtain finite-time bounds on the quality of the solution obtained by our algorithms. As mentioned earlier, this is challenging as, to the best of our knowledge, there are no convergence rate results available for multi-timescale stochastic approximation schemes, and hence, for actor-critic algorithms. 
\bibliography{risk-sensitive-rl}
\bibliographystyle{plainnat}

\end{document}